\documentclass[11pt]{article}

\usepackage[margin=1in]{geometry}
\usepackage[T1]{fontenc}
\usepackage{lmodern}
\usepackage{amsmath,amssymb,amsthm,mathtools}
\usepackage{enumitem,booktabs,microtype}
\usepackage{xcolor}
\usepackage{algorithm}
\usepackage[noend]{algpseudocode}
\usepackage[round]{natbib}
\usepackage{hyperref}
\hypersetup{colorlinks=true,citecolor=blue!70!black,linkcolor=black,urlcolor=blue!70!black}
\usepackage{aliascnt}
\usepackage[capitalise,nameinlink]{cleveref}

\newtheorem{theorem}{Theorem}[section]
\newaliascnt{lemma}{theorem}
\newtheorem{lemma}[lemma]{Lemma}
\aliascntresetthe{lemma}
\newaliascnt{proposition}{theorem}
\newtheorem{proposition}[proposition]{Proposition}
\aliascntresetthe{proposition}
\newaliascnt{corollary}{theorem}
\newtheorem{corollary}[corollary]{Corollary}
\aliascntresetthe{corollary}
\theoremstyle{definition}
\newaliascnt{definition}{theorem}
\newtheorem{definition}[definition]{Definition}
\aliascntresetthe{definition}
\theoremstyle{remark}
\newaliascnt{remark}{theorem}
\newtheorem{remark}[remark]{Remark}
\aliascntresetthe{remark}

\crefname{theorem}{Theorem}{Theorems}
\crefname{lemma}{Lemma}{Lemmas}
\crefname{proposition}{Proposition}{Propositions}
\crefname{corollary}{Corollary}{Corollaries}
\crefname{definition}{Definition}{Definitions}
\crefname{remark}{Remark}{Remarks}
\crefname{algorithm}{Algorithm}{Algorithms}
\crefname{equation}{}{}
\Crefname{equation}{Equation}{Equations}
\crefname{section}{Section}{Sections}
\crefname{subsection}{Section}{Sections}
\crefname{subsubsection}{Section}{Sections}
\crefname{appendix}{Appendix}{Appendices}

\newenvironment{restatement}[1]%
  {\par\medskip\noindent\textbf{#1, restated.}\ \itshape}%
  {\par\medskip}

\newcommand{\cX}{\mathcal{X}}
\newcommand{\cH}{\mathcal{H}}
\newcommand{\cG}{\mathcal{G}}
\newcommand{\cD}{\mathcal{D}}
\newcommand{\cF}{\mathcal{F}}
\newcommand{\cA}{\mathcal{A}}
\newcommand{\cC}{\mathcal{C}}
\newcommand{\cP}{\mathcal{P}}
\newcommand{\cT}{\mathcal{T}}
\newcommand{\cU}{\mathcal{U}}
\newcommand{\Err}{\mathrm{Err}}
\newcommand{\E}{\mathbb{E}}
\newcommand{\R}{\mathbb{R}}
\newcommand{\VC}{\mathrm{VC}}
\newcommand{\wtO}{\widetilde{O}}
\newcommand{\Reg}{\mathrm{Reg}}
\newcommand{\bF}{\overline{F}}
\newcommand{\Fp}{\cF_{\cG,\cH}}
\newcommand{\Qbar}{\overline{Q}}
\newcommand{\phat}{\widehat{p}}
\newcommand{\qhat}{\widehat{q}}
\newcommand{\Yhat}{\widehat{Y}}
\newcommand{\Dhat}{\widehat{\Delta}}
\newcommand{\ncf}{n_{\mathrm{conf}}}
\newcommand{\npt}{n_{\mathrm{part}}}
\newcommand{\Prest}{\cP_{\mathrm{rest}}}

\newcommand{\BC}{\mathrm{BC}}
\newcommand{\TV}{\mathrm{TV}}
\newcommand{\gs}{\mathsf{G}}
\DeclareMathOperator{\sgn}{sgn}

\title{\bf Sample-Optimal Multi-Group Learning with an ERM Oracle}
\author{Princewill Okoroafor\\Harvard University}
\date{}

\begin{document}
\maketitle

\begin{abstract}
In multi-group learning, a single predictor must perform nearly as well as
the best hypothesis in a class on every group in a rich, possibly infinite
and overlapping, collection of subpopulations.  The known algorithms with the
optimal sample complexity keep track of every pair of a group and a
hypothesis, which is impossible for the large classes of interest, and the
known oracle-based algorithms have a worse rate.  We
ask whether the optimal rate can be attained by a learner that accesses the
classes only through an empirical risk minimization (ERM) oracle.  We give
an algorithm that uses $\wtO((d_G+d_H)/\varepsilon^2)$ samples and
polynomially many calls to a weighted ERM oracle over pairs of a group and a
hypothesis, where $d_G$ and $d_H$ are the VC dimensions of the group and
hypothesis classes and $\varepsilon$ bounds the excess risk on every group,
weighted by its probability.  We also show a matching lower bound of
$\Omega((d_G+d_H)/\varepsilon^2)$ samples for every learner even with
unlimited computation, so oracle efficiency costs no additional samples.  We also consider a natural special case in which every input
belongs to at most $k$ groups and these groups can be listed.  In this case,
an ERM oracle for the hypothesis class alone suffices, with
$\wtO((d_H+\log k)/\varepsilon^2)$ samples in the agnostic setting and
$\wtO((d_H+\log k)/\varepsilon)$ samples when every group has a hypothesis
that is correct on it.

\end{abstract}

\section{Introduction}
\label{sec:intro}

Consider a model that predicts whether a patient will have an adverse
outcome.  A guarantee on its average error over all patients says little
about older patients, patients from a particular clinic, or patients at the
intersection of several such attributes.  On one of these groups, a
different model from the same class might be far more accurate.  Training a
separate model for each group fails as soon as the groups overlap or are too
numerous to list, which is the typical situation when the groups are
described by, say, all small decision trees over the patient features.

Multi-group learning \citep{rothblum2021multigroup,tosh2022simple} asks for one
predictor that competes, on every group simultaneously, with the best model
for that group.  Write $\cG$ for a class of groups, each given by its
indicator $g:\cX\to\{0,1\}$, and $\cH$ for a class of binary classifiers
$h:\cX\to\{-1,+1\}$.  For a loss $\ell$ and a distribution over inputs $X$
and labels $Y\in\{-1,+1\}$, we ask for a predictor $Q$ such that
\begin{equation}
  \label{eq:intro-goal}
  \sup_{g\in\cG,\,h\in\cH}\
  \E\Bigl[g(X)\bigl(\ell(Q(X),Y)-\ell(h(X),Y)\bigr)\Bigr]
  \ \le\ \varepsilon .
\end{equation}
Each term compares $Q$ with a comparator $h$ on the inputs in group $g$,
weighted by the probability of the group; different groups may have
different best comparators, and for this reason $Q$ is allowed to lie
outside $\cH$.  A bound of $\varepsilon$ in~\eqref{eq:intro-goal} gives
excess risk at most $\varepsilon/\gamma$ conditionally on every group of
probability at least $\gamma$.

Three properties matter for an algorithm that learns such a predictor.  The
first is \emph{data}: the number of labeled samples should be as small as
possible.  The second is \emph{computation}: the classes $\cG$ and $\cH$
are usually exponentially large or infinite, so the algorithm should access
them only through a learning primitive, an oracle for weighted empirical
risk minimization (ERM), with polynomially many calls.  The third is
\emph{determinism}: the output should be a fixed function of the input, so
that the same patient always receives the same prediction.

For finite classes, \citet{tosh2022simple} attain the optimal
$\varepsilon^{-2}$ dependence, with a logarithmic dependence on the sizes of
the classes, by reducing to sleeping experts as \citet{blum2020advancing} did
online, and \citet{haghtalab2023unifying} attain it with simpler game
dynamics, assuming that some mixture of hypotheses is multi-group optimal.
Both algorithms keep a weight on every pair of a group and a hypothesis.  For
classes of finite VC dimension, \citet{tosh2022simple} also give a simple
algorithm that outputs a deterministic decision list, given an oracle that
finds a group--hypothesis pair improving on the current list, but it needs
about $\varepsilon^{-3}$ samples; the closely related algorithm of
\citet{globusharris2022algorithmic} has a similar rate.  In the online
setting, \citet{deng2024groupwise} give oracle-efficient algorithms with an
oracle over pairs for smoothed, transductive, and related settings, but none
of these yields the desired batch guarantee.  Their $\sqrt T$-type bound
needs polynomially many samples per round from a known base measure with
respect to which the inputs are smooth, which for independent inputs and the
best rate is the unknown input distribution itself; their transductive bound
needs the set of possible inputs in advance and finite classes.  With
independent inputs and only the training sample, their oracle-efficient
algorithm has regret of order $T^{2/3}$, which converts into about
$\varepsilon^{-3}$ samples.  Their bounds also hold in expectation for each
group separately, whereas a batch guarantee for all groups at once needs a
bound that holds with high probability uniformly over groups.  This leaves
the following questions open.
\begin{enumerate}[label=(\arabic*),leftmargin=2em]
\item \emph{Can a learner that accesses the classes only through an ERM
  oracle attain the optimal sample complexity of multi-group learning?}
\item \emph{Can it do so with a single deterministic predictor?}
\item \emph{What kind of ERM oracle does such a learner need?}
\end{enumerate}

We address all three.  For the first, we give an algorithm that uses
$\wtO((d_G+d_H+\log(1/\delta))/\varepsilon^2)$ samples and polynomially many
calls to a weighted ERM oracle, and we show that, for every pair of classes
satisfying mild closure conditions, every learner, regardless of
computation, needs this many samples up to logarithmic factors.  For the
second, yes when the inputs carry a computable order: the output can then be
made deterministic with the same sample complexity and no further oracle
calls.  For the third, our oracle optimizes over pairs of a group and a
hypothesis; we show that such an oracle cannot in general be implemented from
separate ERM oracles for the two classes.  We then identify a natural
assumption, bounded membership, under which an ERM oracle for the hypothesis
class alone suffices: when every input belongs to at most $k$ groups and
these groups can be listed, an algorithm with polynomially many ERM calls
over $\cH$ uses $\wtO((d_H+\log k+\log(1/\delta))/\varepsilon^2)$ samples,
which is again optimal.  Under \emph{group realizability}
\citep{ardeshir2026group}, where the labels are a function of the inputs
that agrees on each group with some hypothesis in $\cH$, it attains the
faster rate $\wtO((d_H+\log k+\log(1/\delta))/\varepsilon)$ with a
deterministic output; this makes an approach of \citet{ardeshir2026group},
which runs in polynomial time when there are polynomially many groups, work
for infinite classes of groups with bounded membership.  \Cref{sec:related} discusses
related work in detail.

\subsection{Main results}
\label{sec:main-results}

Let $\cG\subseteq\{0,1\}^{\cX}$ and $\cH\subseteq\{-1,+1\}^{\cX}$ have VC
dimensions $d_G$ and $d_H$, with $d=d_G+d_H\ge1$, and let
$\ell:\{-1,+1\}\times\{-1,+1\}\to[0,1]$ be a loss.  A \emph{randomized
predictor} $Q$ assigns to each input $x$ a distribution $Q_x$ over
$\{-1,+1\}$ and draws a fresh label from $Q_x$ at each query.  Its
\emph{multi-group excess risk} $\Err(Q)$ is the left side
of~\eqref{eq:intro-goal}, with $\ell(Q(X),Y)$ averaged over the draw
(\cref{def:excess}); a deterministic predictor is the special case of point
masses.

Our oracle optimizes over pairs.  For a pair $(g,h)$ and an input $x$, the
pair's output $(g(x),h(x))$ is one of four symbols in
$\cA=\{0,1\}\times\{-1,+1\}$.  The \emph{pair class} $\Fp$ consists of the
functions $F_{g,h}$ on $\cX\times\cA$ that indicate whether the pair's
output at $x$ equals a given symbol:
$F_{g,h}(x,a)=+1$ if $(g(x),h(x))=a$ and $-1$ otherwise
(\cref{def:pair-class}).  A \emph{weighted ERM oracle} for a class receives
points with real weights and returns a function in the class maximizing the
weighted sum of its values (\cref{def:erm}).  When $\cG$ contains only the
whole space, weighted ERM over $\Fp$ is ordinary weighted ERM over $\cH$.

In our theorems, samples are independent draws from an arbitrary unknown
distribution over $\cX\times\{-1,+1\}$, and a learner must succeed for every
such distribution.  The notations $\wtO$ and $\widetilde\Omega$ hide factors
polylogarithmic in $d$, $1/\varepsilon$, and $1/\delta$, and, in
\cref{thm:membership}, in $k$.

\begin{theorem}[Optimal multi-group learning]
\label{thm:main}
Let $\varepsilon,\delta\in(0,1)$.
\begin{enumerate}[label=(\alph*),leftmargin=2em]
\item \emph{Upper bound.} There is an algorithm (\cref{alg:batch}) that,
  given $\varepsilon$, $\delta$, and a weighted ERM oracle for $\Fp$, uses
  $\wtO((d_G+d_H+\log(1/\delta))/\varepsilon^{2})$ samples and
  $\wtO((d_G+d_H+\log(1/\delta))^2/\varepsilon^{4})$ oracle calls, and with
  probability at least $1-\delta$ outputs a randomized predictor with
  multi-group excess risk at most $\varepsilon$.  If $\cX$ carries a
  computable total order (\cref{rem:order}), the output can instead be a
  deterministic predictor (\cref{alg:det}), with the same guarantees.
\item \emph{Lower bound.} There is an absolute constant $c>0$ such that, if
  $\varepsilon\le10^{-3}$ and $\delta\le1/16$, the following holds for all
  classes $\cG$ and $\cH$ of VC dimensions $d_G,d_H\ge1$ such that $\cG$
  contains the whole space and is closed under complements and $\cH$
  contains a hypothesis and its negation.  For the zero-one loss, every
  learner that with probability at least $1-\delta$ outputs a randomized
  predictor with multi-group excess risk at most $\varepsilon$ uses at least
  $c\,(d_G+d_H+\log(1/\delta))/\varepsilon^{2}$ samples.
\end{enumerate}
\end{theorem}

Part~(b) holds for learners with unlimited computation, so part~(a) is
optimal up to polylogarithmic factors and the restriction to an ERM oracle
costs no additional samples.  Part~(b) follows by reduction from the classical lower bound
for agnostic learning \citep[Theorem~6.8]{shalev2014understanding}
(\cref{sec:lower}).  Its conditions on the classes cannot be dropped
entirely: if $\cH$ contains a single hypothesis, predicting with it gives
multi-group excess risk $0$ without any samples.  This contrasts with multicalibration, a closely related
multi-group guarantee, where, for a range of dimensions, learners that
access the class only through polynomially many ERM calls need a factor
$1/\varepsilon$ more samples than learners with unlimited computation
\citep{collina2026sample,okoroafor2026sample}.  The output of part~(a) is
efficient to evaluate: the randomized predictor is the uniform average of
$n=\wtO((d+\log(1/\delta))/\varepsilon^2)$ simpler predictors, each defined by
$n$ pairs returned by the oracle, and a prediction is drawn by evaluating one
of them, chosen uniformly at random.  The deterministic output also stores
$\wtO((d_G+\log(1/\delta))/\varepsilon^2)$ sample inputs and uses the order,
and evaluating it takes $n^2$ evaluations of stored pairs
(\cref{rem:representation}).  For groups of probability at least
$\gamma$, the same algorithm with accuracy $\gamma\varepsilon$, hence
$\wtO((d+\log(1/\delta))/(\gamma^2\varepsilon^2))$ samples, gives excess risk
at most $\varepsilon$ conditionally on every such group
(\cref{cor:conditional}).  All our upper bounds also hold
with an approximate oracle, one whose answer is within $1$ of the optimal
weighted sum (\cref{rem:approx}).

Weighted ERM over $\Fp$ optimizes over a group and a hypothesis jointly, and
one may ask whether it can be implemented from separate weighted ERM oracles
for $\cG$ and $\cH$.  The following lemma shows that it cannot in general.

\begin{lemma}[Separate oracles do not implement the pair oracle]
\label{lem:oracle}
There is a polynomial-time map that takes a graph on $n$ vertices to classes
$\cG$ and $\cH$ on a finite input space, each of VC dimension at most $n$
and each admitting weighted ERM in linear time, together with a query to a
weighted ERM oracle for $\Fp$ whose optimal value determines the size of the
maximum cut of the graph.  Consequently, unless $\mathrm{P}=\mathrm{NP}$, no
polynomial-time algorithm implements weighted ERM over $\Fp$ using separate
weighted ERM oracles for $\cG$ and $\cH$.
\end{lemma}

\Cref{lem:oracle} concerns implementing the oracle, not learning: it does not
rule out a multi-group learner that uses only separate oracles and never
solves the joint problem.  Whether such a learner can attain the optimal rate
in general is open (\cref{sec:discussion}).

Our second main result identifies a natural assumption under which the pair oracle
is not needed.  Say that $\cG$ has \emph{membership at most $k$} if every
input belongs to at most $k$ groups.  This holds, for instance, for
intersectional groups defined by agreement on a subset of $a$ attributes,
with $k=2^a$, and for hierarchical groups of depth $k$
\citep{deng2024hierarchical}.  We also assume that the groups containing a
given input can be listed, which is immediate in these examples; an ERM
oracle for $\cG$ alone does not provide this (\cref{rem:listing}).

\begin{theorem}[Bounded membership with an ERM oracle for $\cH$]
\label{thm:membership}
Let $\varepsilon,\delta\in(0,1)$, and suppose that every input belongs to
at most $k$ groups and that the groups containing a given input can be
listed (\cref{def:membership}).
\begin{enumerate}[label=(\alph*),leftmargin=2em]
\item \emph{Agnostic.} There is an algorithm (\cref{alg:heavy-ag}) that, given
  $\varepsilon$, $\delta$, and a weighted ERM oracle for $\cH$, uses
  $\wtO((d_H+\log(k/\delta))/\varepsilon^{2})$ samples and at most
  $2k/\varepsilon$ oracle calls, and with probability at least $1-\delta$
  outputs a randomized predictor with multi-group excess risk at most
  $\varepsilon$.
\item \emph{Group-realizable.} If $\ell$ is the zero-one loss and $\cD$ is
  group-realizable (\cref{def:group-realizable}), then there is an algorithm
  (\cref{alg:heavy-re}) that, given $\varepsilon$, $\delta$, and a weighted
  ERM oracle for $\cH$, uses
  $\wtO((d_H+\log(k/\delta))/\varepsilon)$ samples and at most
  $2k/\varepsilon$ oracle calls, and with probability at least $1-\delta$
  outputs a deterministic predictor with multi-group excess risk at most
  $\varepsilon$.
\end{enumerate}
\end{theorem}

Membership at most $k$ forces $d_G\le1+\log_2k$.  Conversely, for a
partition of $\cX$ into $d_G$ cells, the class of all unions of cells contains
the whole space, is closed under complements, has VC dimension $d_G$, and has
membership $2^{d_G-1}$; by \cref{thm:main}(b), part~(a) is therefore optimal
up to polylogarithmic factors.  Run with a probability threshold
$\gamma$, the same algorithm gives excess risk at most $\varepsilon$
conditionally on every group of probability at least $\gamma$ with
$\wtO((d_H+\log(k/\delta))/(\gamma\varepsilon^2))$ samples, and
$\wtO((d_H+\log(k/\delta))/(\gamma\varepsilon))$ under group realizability
(\cref{thm:heavy-agnostic,thm:heavy-realizable}); the dependence $1/\gamma$
is the rate of ERM on a single group of probability $\gamma$, in contrast
with the factor $1/\gamma^2$ of \cref{cor:conditional}.  The hard
instance of \cref{lem:oracle} has membership $2^n$, so \cref{thm:membership}
does not apply to it.  On ordered input spaces, the rounding of
\cref{sec:det} also makes the agnostic predictor of part~(a) deterministic,
with $\wtO((\log k+\log(1/\delta))/\varepsilon^2)$ further samples.

\subsection{Related work}
\label{sec:related}

\paragraph{Multi-group learning.}
\citet{rothblum2021multigroup} formalized multi-group agnostic learning and,
for losses satisfying a compatibility condition, gave an algorithm, through
outcome indistinguishability \citep{dwork2021outcome}, whose sample
complexity is logarithmic in the number of groups.
\citet{blum2020advancing} studied the online version through sleeping
experts \citep{blum1997empirical,freund1997using}.  \citet{tosh2022simple}
reduced the batch problem to sleeping experts, using the algorithm of
\citet{blum2007external} as analyzed by \citet{gaillard2014second}, and an
online-to-batch conversion \citep{cesabianchi2004generalization}, obtaining conditional
excess risk of order $\sqrt{\log(|\cG||\cH|)/n_g}$ on every group, where
$n_g$ is the number of samples in the group; this adapts to the size of each
group, but requires finite classes and enumerates them.  They also gave
\textsc{Prepend}, which outputs a decision list of groups and hypotheses,
applies to VC classes, and is efficient given an oracle that finds a
group--hypothesis pair improving on the current list, with a rate of about
$\varepsilon^{-3}$; the \textsc{ListUpdate} algorithm of
\citet{globusharris2022algorithmic} is closely related.
\citet{haghtalab2023unifying} attain the optimal rate for finite classes with
two copies of Hedge, assuming that some mixture of hypotheses is multi-group
optimal, and
\citet{deng2024hierarchical} obtain deterministic decision trees for
hierarchically structured groups, a setting with bounded membership.
\citet{tosh2022simple} also note that, for finite classes of groups, one can
train one hypothesis per group by ERM and aggregate these with sleeping
experts.  Under group realizability, where every group has a hypothesis that
is correct on it, \citet{ardeshir2026group} obtain fast rates by ERM over the
class of group-realizable concepts, show that this ERM is NP-hard, and
propose instead to train one consistent hypothesis per group and aggregate
with sleeping experts, or with weighted majority \citep{littlestone1994weighted}
for a deterministic output; this runs in polynomial time when there are
polynomially many groups.  \Cref{thm:membership} makes these approaches
efficient for infinite classes of groups with bounded membership, and in the
group-realizable case it addresses, under this assumption, the problem left
open by \citet{ardeshir2026group} of an oracle-efficient algorithm with the
fast rate.  Multi-group learning has also been studied with active label
queries \citep{rittler2023agnostic} and in the transductive setting
\citep{bergam2026price}.
Related multi-group guarantees include subgroup fairness
\citep{kearns2018preventing}, multiaccuracy \citep{kim2019multiaccuracy}, and
multicalibration \citep{hebert2018multicalibration}, whose sample complexity
with an ERM oracle is determined in \citet{okoroafor2026sample}.
\Cref{thm:main} gives the optimal rate for VC classes with polynomially many
oracle calls; our predictor is randomized or, on ordered input spaces,
deterministic, but its guarantee does not adapt to the size of each group.

\paragraph{Oracle-efficient online and hybrid learning.}
In the fully adversarial online setting, no algorithm with access to an
optimization oracle can have vanishing regret in time much smaller than the
square root of the number of experts \citep{hazan2016computational}, which
motivates relaxations of the adversary such as smoothed inputs
\citep{block2022smoothed}.  In the hybrid setting of
\citet{lazaric2009hybrid}, inputs are independent draws from an unknown
distribution while the feedback is adversarial; \citet{wu2024oracle} gave
the first oracle-efficient algorithm in this setting, with regret
$\wtO(T^{3/4})$, and \citet{okoroafor2026hybrid} studied hybrid learners
with constrained adversaries.
\citet{buzaglo2026oracle} showed that Gaussian follow-the-perturbed-leader
\citep{kalai2005efficient} has expected regret $\wtO(\sqrt{Td})$ in this
setting with one ERM call per round, and \citet{okoroafor2026sample} proved a
high-probability version, which we use as a black box.
Online multi-group guarantees with oracles also appear in online
multicalibration \citep{garg2024oracle}.  For online multi-group learning,
\citet{acharya2024oracle} avoid
enumerating the hypotheses by running a meta-algorithm on top of an
oracle-efficient online learner for $\cH$, but they still enumerate the
groups.  \citet{deng2024groupwise} give algorithms that are oracle-efficient
in both classes, in the smoothed, transductive, and other settings.
Their algorithm, following the framework of \citet{lee2022online}, plays a
game between a player over pairs of a group and a hypothesis, implemented by follow-the-perturbed-leader with an oracle that
optimizes jointly over pairs and approximated by repeated draws, and a
learner who answers each round by solving a small linear program.  Our
online algorithm has the same game structure, and our pair oracle is a
special case of their joint oracle.  The difference is the pair player, and
it is what makes a batch guarantee possible.  Their follow-the-perturbed-leader
perturbs with hallucinated inputs drawn from a base measure with respect to
which the inputs are smooth; for independent inputs and the best rate, this
measure is the input distribution, so their $\wtO(\sqrt{dT})$ bound needs
polynomially many fresh draws from it per round, and it holds in expectation
for each group.  Without such access, their oracle-efficient algorithm for
independent inputs reruns \textsc{ListUpdate} on each prefix and has regret
of order $T^{2/3}$.  Our pair player is Gaussian follow-the-perturbed-leader over the
pair class \citep{buzaglo2026oracle}, which perturbs only the past
observations, so it needs no samples beyond the training data, and its
regret bound holds with high probability \citep{okoroafor2026sample}, which
makes it uniform over groups.  Both properties are needed for the batch
conversion to give the $\varepsilon^{-2}$ rate.

\paragraph{Joint optimization over groups and hypotheses.}
\citet{globusharris2022algorithmic} and \citet{deng2024groupwise} note that
alternating between separate oracles for groups and for hypotheses finds
only a local optimum of the joint problem.  \citet{ardeshir2026group} show
that ERM over the class of group-realizable concepts is NP-hard even with an
oracle for consistency over $\cH$, and \citet{kearns2018preventing} show that
auditing for subgroup fairness is computationally equivalent to weak
agnostic learning of the groups.  More classically, consistency with an
intersection of two halfspaces is NP-hard although consistency with one
halfspace is easy \citep{blum1992training}.  To our knowledge,
\cref{lem:oracle} is the first statement that the joint problem itself,
exact optimization over pairs, can be NP-hard even when each class admits
linear-time ERM.

\paragraph{Deterministic predictors.}
\citet{noarov2026optimal} showed that randomization can be removed from
multicalibrated predictors by rounding with one random seed per cell of a
partition of the inputs, built from a computable order, and by treating
repeated inputs separately, and
\citet{okoroafor2026sample} analyzed this rounding for VC classes of tests.
We adapt that analysis to multi-group excess risk, where the quantity to be
rounded is the difference between the conditional losses of the two labels.

\subsection{Technical overview}
\label{sec:overview}

We describe the ideas behind the two main theorems in turn.

\subsubsection{\texorpdfstring{Optimal multi-group learning (\cref{thm:main})}{Optimal multi-group learning (Theorem 1.1)}}
\label{sec:overview-main}

\paragraph{Obstacles.}
The optimal algorithms of \citet{tosh2022simple} weigh every
group--hypothesis pair, which an oracle does not allow.  Their rate comes
from converting an online algorithm into a batch one, and online learning is
hard with only an optimization oracle.  Finally, the natural output of the
conversion is randomized.

\paragraph{The upper bound.}
A natural approach views multi-group learning as a game.  On each round, the
learner predicts a label at the current input, and a pair player chooses a
pair $(g,h)$; the player's payoff is the amount by which $h$ beats the
learner on the input, counted only if the input is in group $g$.  If the
player has small regret and the learner keeps the player's expected payoff
nonpositive on every round, then no pair beats the learner by much in
hindsight, which is a small multi-group regret.  The pair player must be
implemented with an oracle.  Against adversarial inputs this is hard in
general: no algorithm with an optimization oracle has vanishing regret in
time much smaller than the square root of the size of the class
\citep{hazan2016computational}.  Our samples are independent, so the inputs
are random draws from a fixed distribution and only the payoffs are
adversarial.  In this \emph{hybrid} setting \citep{lazaric2009hybrid}, the
Gaussian follow-the-perturbed-leader algorithm of \citet{buzaglo2026oracle}
makes one ERM call per round, and \citet{okoroafor2026sample} showed that its
regret is $\wtO(\sqrt{Td})$ with high probability.  This choice of pair
player is essential for the batch conversion (\cref{sec:related}).

That algorithm competes with a class of $\pm1$-valued functions, with a
payoff equal to a scalar times the function's value.  The payoff of a pair is
not of this form: it depends on two bits, whether the input is in the group
and which label the hypothesis predicts.  We therefore append to each input a
uniformly random symbol $a$ out of the four possible outputs of a pair, and
run it over the functions that indicate whether the pair's output equals
$a$.  Averaged over $a$, the payoff of such a function is half
the payoff of the pair plus a term that is the same for all pairs, so low
regret over these functions is low regret over pairs (\cref{lem:pair-regret}).

The learner keeps the player's payoff nonpositive by a voting rule.  It draws
$M$ pairs from the player's current distribution ($M=T$ suffices) and, at an
input, predicts the label of a uniformly random drawn pair whose group
contains the input.  Against the drawn pairs, the expected payoff of this
prediction is exactly zero whatever the true label, and the sampling error is
small
(\cref{lem:voting}).  The online multi-group regret is therefore
$\wtO(\sqrt{T(d+\log(1/\delta))})$ (\cref{thm:online}).  Averaging the
learner's predictors over the rounds gives a batch predictor.  Each round's
predictor is fixed before the round's example is seen, so its loss on that
example is an unbiased estimate of its population loss, and a martingale
argument over a cover of $\cG$, used only in the analysis, gives the batch
rate $\wtO((d+\log(1/\delta))/\varepsilon^2)$ (\cref{thm:batch}).

To make the output deterministic, we fix the random choices of the averaged
predictor once per cell of a partition of the inputs into cells of small
probability, following \citet{noarov2026optimal} and
\citet{okoroafor2026sample}.  Rounding can fail at inputs of large
probability, where a single random choice is not averaged out.  At such an
input, the only quantity that matters is the difference between the
conditional losses of the two labels.  We estimate it from a separate sample;
where it is confidently nonzero we predict the better label, and elsewhere it
is small, so rounding costs little (\cref{thm:det}).

\paragraph{The lower bound.}
Part~(b) holds for all learners, and we prove it by reduction to agnostic
learning.  If the groups include the whole space, a multi-group learner is
an agnostic learner for $\cH$, which gives the terms $d_H$ and
$\log(1/\delta)$.  For the term $d_G$, fix a hypothesis $h$ whose negation is
also in $\cH$.  The classifiers that agree with $h$ on a group and with $-h$
off it form a class of VC dimension $d_G$, and, when the groups are closed
under complements, a multi-group learner with accuracy $\varepsilon$ is an
agnostic learner for this class with accuracy $2\varepsilon$.  The classical
lower bound for agnostic learning then applies.  Our learners output
randomized predictors, so we include a proof of that lower bound in this
form, by Assouad's lemma and a two-point argument.

\subsubsection{\texorpdfstring{The oracle question and bounded membership (\cref{lem:oracle,thm:membership})}{The oracle question and bounded membership (Lemma 1.2 and Theorem 1.3)}}
\label{sec:overview-membership}

\paragraph{Why separate oracles do not suffice (\cref{lem:oracle}).}
Weighted ERM over the pair class can express the bilinear objective
$\sum_rw_rg(x_r)h(x_r)$ over pairs (\cref{lem:bilinear-pair} in \cref{app:oracle}).  Fixing either
the group or the hypothesis makes this objective linear in the other, so
separate oracles can only improve one coordinate at a time, which suggests
that they do not suffice.  We make this precise by encoding \textsc{Max-Cut}: the group
selects a sign for each vertex, the hypothesis is forced to copy these signs,
and the objective counts the edges whose endpoints receive different signs,
while each class on its own decomposes across vertices and admits linear-time
ERM.

\paragraph{Bounded membership.}
A group of probability below $\varepsilon$ contributes at most $\varepsilon$
to the multi-group excess risk, so only groups of probability at least
$\gamma=\varepsilon$ matter.  If every input belongs to at most $k$ groups,
then the probabilities of any collection of distinct groups sum to at most
$k$, so at most $k/\gamma$ groups have probability at least $\gamma$.  These heavy groups can be found
from a reference sample: list the groups containing each sample point and
keep those that appear often (\cref{lem:heavy}).  With finitely many
relevant groups in hand, we follow \citet{tosh2022simple} and
\citet{ardeshir2026group}, who train one hypothesis per group by ERM and then
aggregate the trained hypotheses with a sleeping-experts algorithm, in which
a hypothesis is awake on the inputs of its group.  Training takes one ERM
call over $\cH$ per heavy group, at most $2k/\varepsilon$ calls in all, and
no optimization over pairs is needed, because the groups are now explicit.  In
the agnostic case the aggregation is a sleeping version of the Prod
algorithm \citep{cesabianchi2007improved}, and the average of its predictors is the output.  Under group
realizability, each trained hypothesis errs on a small part of its group, a
deterministic weighted majority \citep{littlestone1994weighted} over the
trained hypotheses makes few
mistakes on every group, and the majority vote of its predictors over the
rounds is a deterministic predictor with the fast rate.

\section{Preliminaries}
\label{sec:prelims}

\paragraph{Predictors and risks.}
Let $\cD$ be a distribution over $\cX\times\{-1,+1\}$ with marginal $\cD_X$
on inputs, and let $\ell:\{-1,+1\}\times\{-1,+1\}\to[0,1]$ be a loss.  A
randomized predictor $Q$ maps each input $x$ to a distribution $Q_x$ on
$\{-1,+1\}$; it is used by drawing a fresh label $\Yhat\sim Q_X$, independent
of $Y$ given $X$, at each query.  A deterministic predictor
$f:\cX\to\{-1,+1\}$ is identified with the randomized predictor
$Q_x=\delta_{f(x)}$.

\begin{definition}[Multi-group excess risk]
\label{def:excess}
For a randomized predictor $Q$, a group $g\in\cG$, and a hypothesis
$h\in\cH$, let
\[
  L(Q;g)=\E_{(X,Y)\sim\cD}\Bigl[g(X)\sum_{a\in\{-1,+1\}}Q_X(a)\,\ell(a,Y)\Bigr],
  \qquad
  L(h;g)=\E_{(X,Y)\sim\cD}\bigl[g(X)\,\ell(h(X),Y)\bigr].
\]
The \emph{multi-group excess risk} of $Q$ is
\[
  \Err(Q)=\Err(Q;\cG,\cH,\cD)=\sup_{g\in\cG,\,h\in\cH}\bigl\{L(Q;g)-L(h;g)\bigr\}.
\]
\end{definition}

The supremum lets each group choose its own comparator.  For a group with
$p_g=\Pr\{g(X)=1\}>0$, the \emph{conditional excess risk}
$\Err_{\mathrm{cond}}(Q;g)=(L(Q;g)-\inf_{h\in\cH}L(h;g))/p_g$ is the excess
risk of $Q$ over the best hypothesis among inputs in $g$, and
$\Err_{\mathrm{cond}}(Q;g)\le\Err(Q)/p_g$.  Throughout, $d_G=\VC(\cG)$,
$d_H=\VC(\cH)$, and $d=d_G+d_H\ge1$, and $\Pi_{\cG}(n)$ denotes the growth
function of $\cG$ on $n$ points.  By the Sauer--Shelah lemma
\citep{sauer1972density}, $\Pi_{\cG}(n)\le(en/d_G)^{d_G}$ for $n\ge d_G$, and
similarly for $\cH$.

\paragraph{Oracles.}

\begin{definition}[Weighted ERM oracle]
\label{def:erm}
A \emph{weighted ERM oracle} for a class $\cF\subseteq\{-1,+1\}^{\cU}$
receives a finite list of pairs $(u_i,w_i)\in\cU\times\R$ and returns some
$F\in\cF$ with $\sum_iw_iF(u_i)=\max_{F'\in\cF}\sum_iw_iF'(u_i)$, breaking
ties in an arbitrary way.  For $\xi\ge0$, a \emph{$\xi$-approximate}
weighted ERM oracle may instead return any $F\in\cF$ with
$\sum_iw_iF(u_i)\ge\max_{F'\in\cF}\sum_iw_iF'(u_i)-\xi$.  In both cases, we
assume that the answer depends only on the input list and on randomness
independent of everything else.  Each call counts as one oracle call, and
returned functions can be stored and evaluated at any point.
\end{definition}

For a class of $\{0,1\}$-valued groups, weighted ERM is defined in the same
way.  Maximizing a weighted sum of $\pm1$ values is ordinary weighted
classification with labels $\sgn(w_i)$ and weights $|w_i|$.

\begin{definition}[Pair class]
\label{def:pair-class}
Let $\cA=\{0,1\}\times\{-1,+1\}$.  For $g\in\cG$ and $h\in\cH$, define
$F_{g,h}:\cX\times\cA\to\{-1,+1\}$ by
\[
  F_{g,h}(x,a)=2\,\mathbf 1\bigl\{(g(x),h(x))=a\bigr\}-1,
\]
and let $\Fp=\{F_{g,h}:g\in\cG,\ h\in\cH\}$.  A weighted ERM oracle for $\Fp$
returns, with its answer $F_{g,h}$, the pair $(g,h)$.
\end{definition}

If $\cG=\{\mathbf 1_{\cX}\}$, then $F_{g,h}(x,(1,b))=b\,h(x)$ and
$F_{g,h}(x,(0,b))=-1$, so weighted ERM over $\Fp$ is weighted ERM over $\cH$.

\begin{lemma}[Complexity of the pair class]
\label{lem:pair-vc}
For every $n$, $\Pi_{\Fp}(n)\le\Pi_{\cG}(n)\,\Pi_{\cH}(n)$, and
$\VC(\Fp)\le6d$.
\end{lemma}

\begin{proof}
On points $(x_i,a_i)_{i\le n}$, the values of $F_{g,h}$ are determined by the
labelings $(g(x_i))_i$ and $(h(x_i))_i$, which gives the first claim.  If
$\Fp$ shatters $n\ge d$ points, then by the Sauer--Shelah lemma and the
concavity of $z\mapsto z\log(en/z)$,
$n\log2\le d_G\log(en/d_G)+d_H\log(en/d_H)\le d\log(2en/d)$, which fails for
$n\ge6d$.
\end{proof}

\paragraph{Hybrid online learning.}
In the hybrid setting, a learner competes with a class
$\cF\subseteq\{-1,+1\}^{\cU}$ for $T$ rounds.  Before round $t$ it holds a
distribution $q_t$ over $\cF$ determined by its earlier observations and
internal randomness.  On round $t$, an input $U_t$ is drawn from a fixed
unknown distribution, independently of everything before; an environment
chooses a feedback value $Z_t\in[-1,1]$ as a function of $U_t$, the past, and
randomness independent of the learner's round-$t$ draw; the learner draws
$F_t\sim q_t$ with fresh randomness; and the learner observes $(U_t,Z_t)$.
Write $\bF_t(u)=\E_{F\sim q_t}[F(u)]$ for the learner's mixed prediction.

The learner we use is Gaussian follow-the-perturbed-leader
\citep{kalai2005efficient,buzaglo2026oracle}: with scale $\theta=4\sqrt{\log(2T)}$,
a draw from $q_t$ is any maximizer of
$\sum_{s<t}(Z_s+\theta\gs_s)F(U_s)$ over $F\in\cF$, for fresh independent
standard Gaussians $\gs_1,\ldots,\gs_{t-1}$, and is computed by one call to a
weighted ERM oracle for $\cF$.  For an absolute constant $C_0$, let
\[
  \Reg(T,d_{\cF},\zeta)
  =C_0\sqrt{T\bigl(d_{\cF}\log(eT)+\log(1/\zeta)\bigr)}\,\log^2(2T)
  =\wtO\Bigl(\sqrt{T\bigl(d_{\cF}+\log(1/\zeta)\bigr)}\Bigr).
\]

\begin{theorem}[High-probability hybrid regret, \citealp{okoroafor2026sample}]
\label{thm:hybrid}
Let $\cF\subseteq\{-1,+1\}^{\cU}$ have VC dimension at most $d_{\cF}\ge1$,
and let $T\ge1$.  For every input distribution, every environment as above,
and every $\zeta\in(0,1)$, Gaussian follow-the-perturbed-leader satisfies,
with probability at least $1-\zeta$,
\[
  \sup_{F\in\cF}\sum_{t=1}^TZ_t\bigl(F(U_t)-\bF_t(U_t)\bigr)
  \le\Reg(T,d_{\cF},\zeta).
\]
Each draw from $q_t$, including independent draws in addition to $F_t$,
costs one weighted ERM call on at most $T$ points.  The bound also holds,
with a larger absolute constant $C_0$, when every call is answered by a
$\xi$-approximate oracle with $\xi\le1$.
\end{theorem}

This is Theorem~2.5 of \citet{okoroafor2026sample}, which turns the
expected-regret bound of \citet{buzaglo2026oracle} into a high-probability
one and allows approximate oracles.

\begin{remark}[Approximate oracles]
\label{rem:approx}
All our upper bounds hold when the oracle is only $\xi$-approximate with
$\xi\le1$, with larger absolute constants.  In \cref{sec:upper}, the oracle
is used only by the pair player, whose regret bound is \cref{thm:hybrid},
and the rest of the analysis uses only its draws.  The rounding of
\cref{sec:det} makes no oracle calls.  In \cref{sec:membership}, the effect
of the slack on each trained hypothesis is discussed in the paragraph that
describes the algorithms.
The slack is additive on the unnormalized weighted sum, so it is small
compared with the sums that the algorithms optimize.
\end{remark}

\paragraph{Uniform convergence with predictable multipliers.}
The batch analysis uses two uniform deviation bounds; their proofs are in
\cref{app:uniform}.  In the first, the multipliers may depend on earlier
samples, which is the situation after running an online algorithm.

\begin{lemma}[Predictable multipliers]
\label{lem:predictable}
Let $W_t=(X_t,Y_t,R_t)$, $t\le n$, be independent, with
$(X_t,Y_t)\sim\cD$ and $R_t$ uniform on $[0,1]$ independent of $(X_t,Y_t)$.
For each $t$, let $z_t:\cX\times\{-1,+1\}\times[0,1]\to[0,1]$ be determined
by $W_1,\ldots,W_{t-1}$ and by randomness independent of $W_t,\ldots,W_n$.
There is an absolute constant $C$ such that, with probability at least
$1-\zeta$,
\[
  \sup_{g\in\cG}\ \Bigl|\frac1n\sum_{t=1}^n\Bigl(g(X_t)z_t(W_t)
  -\E_W\bigl[g(X)z_t(W)\bigr]\Bigr)\Bigr|
  \le C\sqrt{\frac{(d_G+1)\log(en)+\log(2/\zeta)}{n}},
\]
where in each expectation $z_t$ is held fixed and $W=(X,Y,R)$ is a fresh
draw.
\end{lemma}

\begin{lemma}[Uniform convergence for comparators]
\label{lem:comparators}
For $n$ independent samples from $\cD$, with probability at least $1-\zeta$,
\[
  \sup_{g\in\cG,\,h\in\cH}\ \Bigl|\frac1n\sum_{t=1}^ng(X_t)\,\ell(h(X_t),Y_t)
  -L(h;g)\Bigr|
  \le C\sqrt{\frac{d\log(en)+\log(2/\zeta)}{n}}.
\]
\end{lemma}

\begin{remark}[Order on inputs]
\label{rem:order}
Our deterministic predictor (\cref{sec:det}) uses a computable total order on
$\cX$ whose order intervals are measurable, together with computable
equality tests; the lexicographic order on $\R^n$ suffices.  The order is
used only to partition the input space into cells of small probability.  The
randomized guarantees do not use it.
\end{remark}

\section{Oracle-efficient multi-group learning}
\label{sec:upper}

This section gives the randomized algorithm of \cref{thm:main}(a), proves
its sample complexity, and shows that this sample complexity is optimal.
The algorithm runs an online learner over the training samples, one sample
per round, and averages its predictors.  \Cref{cor:conditional} gives a
conditional guarantee on large groups, and \cref{sec:lower} proves the lower
bound of \cref{thm:main}(b), which shows that \cref{thm:batch} is optimal up
to polylogarithmic factors among all learners.  The
pair-player view of multi-group learning follows \citet{deng2024groupwise}.
The pair player is the hybrid learner of \citet{buzaglo2026oracle}, with the
high-probability bound of \citet{okoroafor2026sample}; unlike the pair player
of \citet{deng2024groupwise}, it needs no samples of the input distribution
beyond the training data, and its bound holds uniformly over groups with
high probability.  Long proofs are in \cref{app:upper}.

\paragraph{The online game.}
Consider an online sequence of inputs $X_t$, labels $Y_t$, and the learner's
predictions $\Yhat_t$.  The \emph{payoff} of a pair $(g,h)$ on round $t$ is
\[
  \rho_t(g,h)=g(X_t)\bigl(\ell(\Yhat_t,Y_t)-\ell(h(X_t),Y_t)\bigr),
\]
the amount by which $h$ beats the learner on the input, counted only if the
input is in $g$.  The sum $\sup_{g,h}\sum_t\rho_t(g,h)$ is the learner's
\emph{online multi-group regret}.  The payoff depends on $(g,h)$ only through
the pair's output: $\rho_t(g,h)=\Gamma_t(g(X_t),h(X_t))$ for the payoff table
\[
  \Gamma_t(b,z)=b\bigl(\ell(\Yhat_t,Y_t)-\ell(z,Y_t)\bigr)\in[-1,1],
  \qquad(b,z)\in\cA.
\]
The online algorithm plays a pair player against the learner.  The pair
player is Gaussian follow-the-perturbed-leader over the pair class $\Fp$
(\cref{thm:hybrid}); on each round it receives the input $U_t=(X_t,A_t)$,
for a fresh uniform symbol $A_t\in\cA$, and the feedback $Z_t=\Gamma_t(A_t)$.
The inputs $U_t$ are independent draws from
$\cD_X\otimes\mathrm{Unif}(\cA)$, as the hybrid setting requires.

The random symbol turns the payoff of a pair into the payoff of a single
function in the pair class, up to a term common to all pairs.

\begin{lemma}[Unbiased encoding]
\label{lem:encoding}
For every $\Gamma:\cA\to[-1,1]$, every pair $(g,h)$, and every input $x$, if
$A$ is uniform on $\cA$, then
\[
  \E_A\bigl[\Gamma(A)\,F_{g,h}(x,A)\bigr]
  =\frac12\,\Gamma\bigl(g(x),h(x)\bigr)-\frac14\sum_{a\in\cA}\Gamma(a).
\]
\end{lemma}

\begin{proof}
Exactly one symbol $a$ equals $(g(x),h(x))$, so
$\frac14\sum_a\Gamma(a)(2\,\mathbf 1\{(g(x),h(x))=a\}-1)
=\frac12\Gamma(g(x),h(x))-\frac14\sum_a\Gamma(a)$.
\end{proof}

Applied with $\Gamma=\Gamma_t$, the lemma shows that, averaged over the
symbol, the feedback of a function $F_{g,h}$ is $\frac12\rho_t(g,h)$ minus a
term that does not depend on the pair; this term cancels in regret, so regret
over $\Fp$ is regret over pairs.

\paragraph{The voting rule.}
On round $t$, the algorithm draws $M$ independent pairs
$(g_t^{(i)},h_t^{(i)})$, $i\le M$, from the pair player's current
distribution $q_t$ over pairs, before seeing the round's input.  At an input
$x$, let $N_t(x)=\sum_ig_t^{(i)}(x)$ be the number of drawn pairs whose group
contains $x$.  The prediction rule is
\begin{equation}
  \label{eq:vote}
  \phat_t(z\mid x)=
  \begin{cases}
    \dfrac{1}{N_t(x)}\displaystyle\sum_{i=1}^Mg_t^{(i)}(x)\,
      \mathbf 1\bigl\{h_t^{(i)}(x)=z\bigr\}, & N_t(x)>0,\\[2ex]
    \mathbf 1\{z=+1\}, & N_t(x)=0,
  \end{cases}
  \qquad z\in\{-1,+1\}:
\end{equation}
the label of a uniformly random drawn pair whose group contains $x$, or $+1$
if there is none.  Let $\qhat_t$ be the uniform distribution on the drawn
pairs.  For every input $x$ and label $y$,
\begin{equation}
  \label{eq:vote-zero}
  \E_{\Yhat\sim\phat_t(\cdot\mid x)}\ \E_{(g,h)\sim\qhat_t}
  \Bigl[g(x)\bigl(\ell(\Yhat,y)-\ell(h(x),y)\bigr)\Bigr]=0,
\end{equation}
since both terms equal $\frac1M\sum_ig_t^{(i)}(x)\ell(h_t^{(i)}(x),y)$: the
prediction is distributed as the label of a random drawn pair whose group
contains $x$.

\begin{algorithm}[t]
\caption{$\textsc{PairVote}(T,M)$: online multi-group learning with a pair
  oracle}
\label{alg:online}
\begin{algorithmic}[1]
\Require number of rounds $T$; number $M$ of drawn pairs per round
\State start Gaussian follow-the-perturbed-leader over $\Fp$
  (\cref{thm:hybrid}) with horizon $T$
\For{$t=1,\ldots,T$}
  \State draw $M$ pairs $(g_t^{(i)},h_t^{(i)})\sim q_t$, independently
    \Comment{$M$ oracle calls}
  \State receive $X_t$, draw $\Yhat_t\sim\phat_t(\cdot\mid X_t)$
    by~\eqref{eq:vote} with fresh randomness, and observe $Y_t$
  \State draw $A_t$ uniformly from $\cA$, and give
    $\bigl((X_t,A_t),\,\Gamma_t(A_t)\bigr)$ to the pair player
  \State make the pair player's round draw $F_t\sim q_t$
    \Comment{one oracle call}
\EndFor
\State \Return the drawn pairs, which define $\phat_1,\ldots,\phat_T$
\end{algorithmic}
\end{algorithm}

\begin{algorithm}[t]
\caption{$\textsc{OracleMultiGroup}(\varepsilon,\delta)$: oracle-efficient
  multi-group learning}
\label{alg:batch}
\begin{algorithmic}[1]
\Require accuracy $\varepsilon$; confidence $\delta$; weighted ERM oracle
  for $\Fp$
\State set $n$ as in \cref{thm:batch}
\State run $\textsc{PairVote}(n,n)$ on $n$ independent samples, revealing
  each sample only on its round
\State \Return the randomized predictor
  $\Qbar=\frac1n\sum_t\phat_t$
\end{algorithmic}
\end{algorithm}

The feedback $\Gamma_t(A_t)$ depends on the drawn pairs, the prediction, the
label, and $A_t$, but not on the round draw $F_t$, as the hybrid setting
requires; the round draw is not otherwise used.  A prediction from the output
$\Qbar_x$ of \cref{alg:batch} is drawn by choosing $t$ uniformly from $[n]$
and drawing from $\phat_t(\cdot\mid x)$, which takes $M$ evaluations of
stored pairs and no oracle calls.  The main result of the section is the
following guarantee for this output.

\begin{theorem}[Batch guarantee]
\label{thm:batch}
Run \cref{alg:online} with $T=n\ge2$ and $M=n$ on $n$ independent samples
from $\cD$, and let $\Qbar_x(z)=\frac1n\sum_{t=1}^n\phat_t(z\mid x)$ be the
average of its predictors.  With probability at least $1-\delta$,
\[
  \Err(\Qbar)\le C\log^2(2n)\sqrt{\frac{d\log(en)+\log(8/\delta)}{n}} .
\]
Consequently, $n=\wtO((d+\log(1/\delta))/\varepsilon^2)$ samples suffice for
$\Err(\Qbar)\le\varepsilon$ with probability at least $1-\delta$, with
$n(n+1)=\wtO((d+\log(1/\delta))^2/\varepsilon^4)$ weighted ERM calls over
$\Fp$.
\end{theorem}

\paragraph{Proof outline.}
The proof has two steps.  The first bounds the online multi-group regret of
\cref{alg:online} (\cref{thm:online}).  \Cref{sec:online} combines \cref{lem:encoding},
which turns the regret bound of the hybrid learner over $\Fp$ into a regret
bound over pairs, with the voting rule, which keeps the payoff of the pair
player's mixture near zero.  The second step, in
\cref{sec:batch}, converts the online regret into the batch guarantee: each
round's predictor is fixed before that round's sample is drawn, so the
uniform deviation bounds of \cref{lem:predictable,lem:comparators} replace
the population group losses by realized ones.  The full proofs of
\cref{thm:online,thm:batch} are in \cref{app:upper}.

\subsection{Online regret}
\label{sec:online}

The online multi-group regret of \cref{alg:online} is
$\wtO(\sqrt{T(d+\log(1/\zeta))})$ when the inputs are independent.

\begin{theorem}[Online multi-group regret]
\label{thm:online}
Suppose the inputs $X_t$ are independent draws from $\cD_X$ and each label
$Y_t$ is chosen as a function of the past and $X_t$, and of randomness
independent of the round's prediction, symbol, and round draw.  Let
$T\ge2$, $M\ge1$, and $\zeta\in(0,1)$.  With probability at least $1-\zeta$,
\cref{alg:online} satisfies
\[
  \sup_{g\in\cG,\,h\in\cH}\sum_{t=1}^T\rho_t(g,h)
  \le2\Reg\Bigl(T,6d,\frac\zeta4\Bigr)
  +C\sqrt{T\bigl(d\log(eT)+\log(1/\zeta)\bigr)}
  +CT\sqrt{\frac{\log(32T/\zeta)}{M}} .
\]
With $M=T$, the right side is $\wtO(\sqrt{T(d+\log(1/\zeta))})$.  The
algorithm makes $M+1$ weighted ERM calls over $\Fp$ per round, each on at
most $T$ points.
\end{theorem}

\paragraph{Proof outline.}
The online multi-group regret splits into the regret of the pair player
against its mixture $q_t$ and the mixture's own payoff,
$\sum_t\E_{(g,h)\sim q_t}\rho_t(g,h)$.  For the first, \cref{lem:encoding}
shows that the pair player's regret over pairs is twice the hybrid learner's
regret over $\Fp$, up to two martingales in the random symbols $A_t$; one of
them must be controlled uniformly over pairs, which costs the number of
labelings of the inputs by $\cG$ and $\cH$ (\cref{lem:pair-regret}).
\Cref{thm:hybrid} with $\VC(\Fp)\le6d$ (\cref{lem:pair-vc}) bounds the hybrid
regret.  For the second, \eqref{eq:vote-zero} makes the payoff against the
drawn pairs exactly zero, and the drawn pairs estimate the mixture's group
memberships and labels at the input within $\sqrt{\log(T/\zeta)/M}$
(\cref{lem:voting}).  The full proof is in \cref{app:upper}.

\subsection{From online to batch}
\label{sec:batch}

This subsection proves \cref{thm:batch} from \cref{thm:online}: averaging the
online predictors gives a randomized predictor whose multi-group excess risk
is the online regret divided by the number of samples, plus a uniform
deviation term.

\paragraph{Proof outline of \cref{thm:batch}.}
By linearity, $L(\Qbar;g)-L(h;g)$ is the average over $t$ of the population
group losses of $\phat_t$ minus that of $h$.  The predictor $\phat_t$ is
fixed before the $t$-th sample is drawn, so \cref{lem:predictable}, with the
multiplier $z_t$ equal to the loss of the label drawn from $\phat_t$,
replaces each population group loss of the learner by its realized loss,
uniformly over $g$.  \Cref{lem:comparators} does the same for the
comparators.  The difference of the realized losses is the online
multi-group regret divided by $n$, which \cref{thm:online} bounds.  The full
proof is in \cref{app:upper}.

\begin{corollary}[Conditional guarantee on large groups]
\label{cor:conditional}
Let $\gamma\in(0,1]$.  With
$n=\wtO((d+\log(1/\delta))/(\gamma^2\varepsilon^2))$ samples and
polynomially many weighted ERM calls over $\Fp$, \cref{alg:batch} run with
accuracy $\gamma\varepsilon$ outputs, with probability at least $1-\delta$, a
randomized predictor $\Qbar$ with $\Err_{\mathrm{cond}}(\Qbar;g)\le\varepsilon$
for every $g\in\cG$ with $\Pr\{g(X)=1\}\ge\gamma$.
\end{corollary}

\begin{proof}
By \cref{thm:batch}, $\Err(\Qbar)\le\gamma\varepsilon$, and
$\Err_{\mathrm{cond}}(\Qbar;g)\le\Err(\Qbar)/\gamma$ for every such group.
\end{proof}

\subsection{Lower bound}
\label{sec:lower}

This subsection proves \cref{thm:main}(b) by reducing multi-group learning to
agnostic learning.  For a randomized predictor $Q$ and a classifier $c$,
write $L(Q)=L(Q;\mathbf 1_{\cX})$ and $L(c)=L(c;\mathbf 1_{\cX})$ for the
risks on the whole space, where $L(c;g)$ is defined for any classifier $c$
as in \cref{def:excess}.  We restate the result.

\begin{restatement}{\Cref{thm:main}(b)}
There is an absolute constant $c>0$ such that, if $\varepsilon\le10^{-3}$
and $\delta\le1/16$, the following holds for all classes $\cG$ and $\cH$ of
VC dimensions $d_G,d_H\ge1$ such that $\cG$ contains the whole space and is
closed under complements and $\cH$ contains a hypothesis and its negation.
For the zero-one loss, every learner that with probability at least
$1-\delta$ outputs a randomized predictor with multi-group excess risk at
most $\varepsilon$ uses at least $c\,(d_G+d_H+\log(1/\delta))/\varepsilon^{2}$
samples.
\end{restatement}

\paragraph{Proof outline.}
\Cref{lem:reduction} below shows that a multi-group learner is an agnostic
learner for $\cH$ with the same accuracy, and an agnostic learner with
accuracy $2\varepsilon$ for a class $\cC_h$ of VC dimension $d_G$.  The
classical lower bound for agnostic learning
\citep[Theorem~6.8]{shalev2014understanding} then applies; since our learners
output randomized predictors, we state and prove it in that form
(\cref{prop:lower}).  The proof of \cref{thm:main}(b) at the end of this
subsection combines the two.

\begin{proposition}[Agnostic lower bound for randomized predictors]
\label{prop:lower}
Let $\varepsilon\in(0,1/256]$ and $\delta\in(0,1/16]$, let the loss be the
zero-one loss, and let $\cC\subseteq\{-1,+1\}^{\cX}$ have VC dimension
$m\ge1$.  Every learner that, for every distribution, outputs with
probability at least $1-\delta$ a randomized predictor $Q$ with
$L(Q)-\inf_{c\in\cC}L(c)\le\varepsilon$ uses more than
$m/(2^{20}\varepsilon^2)$ samples and at least
$\log(1/\delta)/(2^{16}\varepsilon^2)$ samples.  This holds for learners with
unlimited computation.
\end{proposition}

\paragraph{Proof outline.}
The hard distributions are supported on a set of $m$ points shattered by
$\cC$, with the uniform distribution on them, and the label at each point
equals a hidden sign with probability $(1+\beta)/2$, where
$\beta=128\varepsilon$.  Some member of $\cC$ predicts every hidden sign, so a
predictor with excess risk at most $\varepsilon$ puts probability at most
$1/128$ on the wrong sign on average over the points, and its majority label
recovers all but $m/64$ of the hidden signs.  Changing one hidden sign
changes the distribution of $n$ samples by little unless
$n\gtrsim m/\varepsilon^2$, and Assouad's lemma
\citep{assouad1983deux,yu1997assouad} turns this into a lower bound on the
expected number of wrongly recovered signs, a contradiction.  On a single
shattered point, two Bernoulli distributions with means $(1\pm\beta)/2$
cannot be distinguished with error $\delta$ from fewer than order
$\log(1/\delta)/\varepsilon^2$ samples.  The full proof is in
\cref{app:lower}.

\begin{lemma}[Reductions to agnostic learning]
\label{lem:reduction}
Let $Q$ be a randomized predictor.
\begin{enumerate}[label=(\alph*),leftmargin=2em]
\item If $\mathbf 1_{\cX}\in\cG$, then $L(Q)-\inf_{h\in\cH}L(h)\le\Err(Q)$.
\item If $\cG$ is closed under complements and $h,-h\in\cH$, let
  $\cC_h=\{x\mapsto h(x)(2g(x)-1):g\in\cG\}$.  Then $\VC(\cC_h)=d_G$ and
  $L(Q)-\inf_{c\in\cC_h}L(c)\le2\Err(Q)$.
\end{enumerate}
\end{lemma}

\begin{proof}
Part~(a) is the definition of $\Err(Q)$ with the group $\mathbf 1_{\cX}$.  For
part~(b), on points $x_1,\ldots,x_n$ the labels of $h(2g-1)$ are
$h(x_i)(2g(x_i)-1)$, which determine and are determined by the labels of
$g$, so $\cC_h$ and $\cG$ shatter the same sets.  The member $c_g=h(2g-1)$
agrees with $h$ on $g$ and with $-h$ on its complement $\bar g\in\cG$, so
\[
  L(Q)-L(c_g)=\bigl(L(Q;g)-L(h;g)\bigr)+\bigl(L(Q;\bar g)-L(-h;\bar g)\bigr)
  \le2\Err(Q). \qedhere
\]
\end{proof}

\begin{proof}[Proof of \cref{thm:main}(b)]
Consider a learner as in part~(b) that uses $n$ samples.
By \cref{lem:reduction}(a), with probability at least $1-\delta$ its output
has excess risk at most $\varepsilon$ with respect to $\cH$, so by
\cref{prop:lower} with $\cC=\cH$, $n>d_H/(2^{20}\varepsilon^2)$ and
$n\ge\log(1/\delta)/(2^{16}\varepsilon^2)$.  By \cref{lem:reduction}(b), its
output has excess risk at most $2\varepsilon\le2\cdot10^{-3}<1/256$ with respect to
$\cC_h$, a class of VC dimension $d_G$, so by \cref{prop:lower},
$n>d_G/(2^{22}\varepsilon^2)$.  Hence
$n\ge(d_G+d_H+\log(1/\delta))/(3\cdot2^{22}\varepsilon^2)$, so part~(b) holds
with $c=1/(3\cdot2^{22})$.
\end{proof}

\section{Making the predictor deterministic}
\label{sec:det}

This section makes the output deterministic when $\cX$ carries a computable
order (\cref{rem:order}), with no further oracle calls and the same sample
complexity (\cref{thm:det}).  It adapts the rounding analysis of
\citet{okoroafor2026sample}, which follows \citet{noarov2026optimal}.
\Cref{alg:det} trains the randomized predictor of \cref{sec:upper}, labels
the inputs of large probability whose better label is clear, and fixes the
remaining random choices once per cell of a partition of the inputs.  The section ends with the proof of \cref{thm:main}(a).

\begin{algorithm}[t]
\caption{$\textsc{DetMultiGroup}(\varepsilon,\delta)$: deterministic
  multi-group learning}
\label{alg:det}
\begin{algorithmic}[1]
\Require accuracy $\varepsilon$; confidence $\delta$; weighted ERM oracle
  for $\Fp$; computable order on $\cX$
\State run $\textsc{OracleMultiGroup}(\varepsilon/2,\delta/4)$ and obtain
  $\Qbar$
\State draw $\ncf$ fresh examples; compute $O$, $\Dhat_x$, and $r_x$ with
  $\delta_0=\delta/8$; form $Q'$ from $\Qbar$ (\cref{sec:det-gap})
\State draw $\npt$ fresh inputs and build the partition $\cP$
  (\cref{sec:det-round})
\State draw one seed $\tau_R$ uniformly from $[0,1]$ for each cell $R\in\cP$
\State \Return the deterministic predictor $f$ of~\eqref{eq:round}
\end{algorithmic}
\end{algorithm}

The main result of the section is the following guarantee for the rounding
steps of \cref{alg:det}.

\begin{theorem}[Deterministic rounding]
\label{thm:det}
Suppose $\cX$ carries an order as in \cref{rem:order}, let $Q$ be a
randomized predictor, and let $f$ be the deterministic predictor that
steps~2--5 of \cref{alg:det} compute from $Q$, using a confidence sample of
size $\ncf$, a partition sample of size $\npt$, both independent of $Q$, and
a failure parameter $\delta_0\in(0,1)$.  With probability at least
$1-4\delta_0$,
\[
  \sup_{g\in\cG}\bigl\{L(f;g)-L(Q;g)\bigr\}
  \le C\sqrt{\Bigl(\frac{\log(e\ncf/\delta_0)}{\ncf}
  +\frac{\log(e\npt/\delta_0)}{\npt}\Bigr)\bigl(d_G+1+\log(2/\delta_0)\bigr)} .
\]
In particular, $\ncf=\npt=\wtO((d_G+\log(1/\delta_0))/\varepsilon^2)$
suffice for the right side to be at most $\varepsilon$.
\end{theorem}

\paragraph{Proof outline.}
\Cref{sec:det-gap} shows that rounding changes the group risks only through
the conditional loss gap, the difference between the conditional losses of
the two labels (\cref{eq:gap}).  At repeated inputs, it estimates this gap
from the confidence sample (\cref{lem:coverage}), predicts the better label
where the gap is confidently nonzero, which cannot increase any group risk,
and shows that the gap is small at the other repeated inputs.
\Cref{sec:det-round} builds the partition and recalls two results of
\citet{okoroafor2026sample}: rounding with one seed per cell has error
controlled by a VC class of tests and a weighted sum of squared cell
probabilities (\cref{lem:cell-tests}), and this sum is small for our
partition (\cref{lem:budget}).  The proof of \cref{thm:det} applies
\cref{lem:cell-tests} to the tests $2g-1$ and the constant $1$, with the
conditional loss gap as the weight; the weight is at most $2r_x$ at the
uncertain repeated inputs and at most $1/2$ elsewhere, so \cref{lem:budget}
bounds the error.  The full proof is in \cref{app:det}.

\subsection{The conditional loss gap}
\label{sec:det-gap}

Rounding changes the group risks only through the difference between the
conditional losses of the two labels, and at inputs of large probability
this difference can be estimated and acted on directly.

For $v\in\{-1,+1\}$, let $\lambda_v(x)=\E[\ell(v,Y)\mid X=x]$, and let
$\Delta(x)=\lambda_{+1}(x)-\lambda_{-1}(x)\in[-1,1]$ be the
\emph{conditional loss gap}.  For a randomized predictor $Q$, let
$m_Q(x)=Q_x(+1)-Q_x(-1)$.  The conditional loss of $Q$ at $x$ is
$\frac12(\lambda_{+1}(x)+\lambda_{-1}(x))+\frac12\Delta(x)m_Q(x)$, so for a
deterministic predictor $f$ and every group $g$,
\begin{equation}
  \label{eq:gap}
  L(f;g)-L(Q;g)=\E\Bigl[g(X)\,\frac{\Delta(X)}2\,\bigl(f(X)-m_Q(X)\bigr)\Bigr].
\end{equation}
The comparator does not appear, so it suffices to control the right side
uniformly over groups.

\paragraph{Confidence intervals at repeated inputs.}
Draw a confidence sample of $\ncf$ labeled examples, independent of all
other data, and let $O$ be the set of inputs that occur in it.  For $x\in O$,
let $N_x$ be its number of occurrences, let $\Dhat_x$ be the average of
$\ell(+1,Y_i)-\ell(-1,Y_i)$ over these occurrences, and let
$r_x=\min\{1,\sqrt{\iota/N_x}\}$ with $\iota=\log(8\ncf/\delta_0)$ for a
failure parameter $\delta_0\in(0,1)$.

\begin{lemma}[Coverage]
\label{lem:coverage}
With probability at least $1-\delta_0$, $|\Dhat_x-\Delta(x)|\le2r_x$ for
every $x\in O$.
\end{lemma}

\begin{proof}
Conditional on the inputs of the confidence sample, the terms averaged in
$\Dhat_x$ are independent, lie in $[-1,1]$, and have mean $\Delta(x)$.  If
$r_x=1$ there is nothing to prove; otherwise Hoeffding's inequality
\citep{hoeffding1963probability} gives
$\Pr\{|\Dhat_x-\Delta(x)|>2r_x\}\le2e^{-2\iota}\le\delta_0/\ncf$.  A union
bound over the at most $\ncf$ elements of $O$ completes the proof.
\end{proof}

Call $x\in O$ \emph{decisive} if $|\Dhat_x|>2r_x$, and let $\mathcal B\subseteq O$ be
the set of decisive inputs.  Given a randomized predictor $Q$, let $Q'$ agree
with $Q$ except that $Q'_x$ is the point mass at $-\sgn(\Dhat_x)$ for
$x\in\mathcal B$.  On the event of \cref{lem:coverage}, $\sgn(\Dhat_x)=\sgn(\Delta(x))$
for $x\in\mathcal B$, so $Q'$ predicts the label of smaller conditional loss there
and
\begin{equation}
  \label{eq:decisive}
  L(Q';g)\le L(Q;g)\qquad\text{for every }g\in\cG;
\end{equation}
and at every $x\in O\setminus\mathcal B$, $|\Delta(x)|\le|\Dhat_x|+2r_x\le4r_x$.

\subsection{Rounding with one seed per cell}
\label{sec:det-round}

The remaining randomness is fixed once per cell of a partition, at a cost
controlled by a VC class of tests and by the sum of squared cell masses
weighted by the loss gap.

\paragraph{The partition.}
Draw $\npt$ unlabeled inputs, independent of all other data, and sort their
distinct values $z_1<\cdots<z_r$ using the order of \cref{rem:order}.  The
cells of $\Prest$ are the sets $\{z_i\}\setminus O$, the open order
intervals $(z_i,z_{i+1})\setminus O$, and the two unbounded intervals before
$z_1$ and after $z_r$, minus $O$; empty sets are discarded.  Together with
the singletons $\{x\}$, $x\in O$, this gives a partition $\cP$ of $\cX$ with
at most $\ncf+2\npt+1$ cells.

\paragraph{Rounding.}
Draw independent seeds $\tau_R$, uniform on $[0,1]$, one for each cell
$R\in\cP$, and let
\begin{equation}
  \label{eq:round}
  f(x)=+1\ \text{ if }\ \tau_{R(x)}\le Q'_x(+1),\qquad f(x)=-1\ \text{ otherwise},
\end{equation}
where $R(x)$ is the cell containing $x$.  The seeds are drawn once, so $f$
is a deterministic function, and $f(x)$ has distribution $Q'_x$ over the
seeds.  The following two results are adapted from
\citet{okoroafor2026sample}; their proofs there apply verbatim with the
label set $\{-1,+1\}$ in place of a grid of prediction values.

\begin{lemma}[Rounding against a VC class of tests, adapted from
  \citealp{okoroafor2026sample}]
\label{lem:cell-tests}
Let $\cP$ be a finite measurable partition of $\cX$, let $Q'$ be a
randomized predictor, and let $f$ be given by~\eqref{eq:round}.  Let $\cT$ be
a class of functions $t:\cX\times\{-1,+1\}\to\{-1,+1\}$ of VC dimension
$d_{\cT}\ge1$, let $\varphi:\cX\times\{-1,+1\}\to\R$ be fixed, and for each
cell $R$ let $\varphi_R\ge|\varphi(x,v)|$ for all $x\in R$ and $v$.  Let
$V_\varphi=\sum_{R\in\cP}\varphi_R^2\,\cD_X(R)^2$.  With probability at
least $1-\delta_0$ over the seeds,
\[
  \sup_{t\in\cT}\Bigl|\E\bigl[t(X,f(X))\,\varphi(X,f(X))\bigr]
  -\E\Bigl[\sum_vQ'_X(v)\,t(X,v)\,\varphi(X,v)\Bigr]\Bigr|
  \le C\sqrt{V_\varphi\bigl(d_{\cT}+\log(2/\delta_0)\bigr)}.
\]
\end{lemma}

\begin{lemma}[Variance budget, adapted from \citealp{noarov2026optimal,okoroafor2026sample}]
\label{lem:budget}
For the construction above, with probability at least $1-2\delta_0$,
\[
  V_{\cP}:=\sum_{x\in O}\cD_X(\{x\})^2r_x^2+\sum_{R\in\Prest}\cD_X(R)^2
  \le C\Bigl(\frac{\log(e\ncf/\delta_0)}{\ncf}
  +\frac{\log(e\npt/\delta_0)}{\npt}\Bigr).
\]
\end{lemma}

We can now prove the upper bound of \cref{thm:main}.

\begin{proof}[Proof of \cref{thm:main}(a)]
\emph{Randomized output.}  This is \cref{thm:batch}.

\emph{Deterministic output.}  Run \cref{alg:det}.  By
\cref{thm:batch} with $\varepsilon/2$ and $\delta/4$, with probability at
least $1-\delta/4$, $\Err(\Qbar)\le\varepsilon/2$.  The confidence and
partition samples are independent of $\Qbar$, so by \cref{thm:det} with
$\delta_0=\delta/8$ and $\varepsilon/2$ in place of $\varepsilon$, with
probability at least $1-\delta/2$, $L(f;g)\le L(\Qbar;g)+\varepsilon/2$ for
every $g$.  Hence, with probability at least $1-3\delta/4$, for every $g$ and
$h$,
$L(f;g)-L(h;g)\le L(\Qbar;g)-L(h;g)+\varepsilon/2\le\varepsilon$.  The
samples number $\wtO((d+\log(1/\delta))/\varepsilon^2)$ in training and
$\ncf+\npt=\wtO((d_G+\log(1/\delta))/\varepsilon^2)$ in rounding, and only
training makes oracle calls.
\end{proof}

\begin{remark}[Representation of the output]
\label{rem:representation}
Every stored group and hypothesis is an oracle output, and hence a member of
$\cG$ or $\cH$, also with an approximate oracle.  The randomized predictor
$\Qbar$ is represented by the $n^2$ drawn pairs, $n$ for each of the $n$
rounds; a draw from $\Qbar_x$ takes $n$ evaluations of pairs, and the
distribution $\Qbar_x$ itself takes $n^2$.  The deterministic predictor $f$
is represented by the same pairs, the seeds, the sorted partition sample,
and the set $O$ with the signs of $\Dhat_x$ on $\mathcal B$; the last two
consist of $\ncf+\npt=\wtO((d_G+\log(1/\delta))/\varepsilon^2)$ sample
inputs.  To evaluate $f$ at an input $x$, one checks whether $x\in\mathcal
B$, locates the cell of $x$ by a lookup in $O$ and a binary search among the
sorted partition inputs, and compares the seed of the cell with
$\Qbar_x(+1)$; this takes $O(\log(\ncf+\npt))$ order comparisons and
equality tests and $n^2=\wtO((d+\log(1/\delta))^2/\varepsilon^4)$
evaluations of pairs.  The representation of $f$ is polynomial in $d$,
$1/\varepsilon$, and $\log(1/\delta)$, but, unlike that of $\Qbar$, it uses
the order on $\cX$ and stored inputs in addition to members of $\cG$ and
$\cH$, and its evaluation computes a distribution rather than a single draw.
\end{remark}

\section{Bounded membership}
\label{sec:membership}

This section proves \cref{thm:membership}: when every input belongs to at
most $k$ groups and these groups can be listed, an ERM oracle for $\cH$
suffices.  The algorithm finds the groups of non-negligible probability,
trains one hypothesis per such group, and aggregates the trained hypotheses
with sleeping experts, following \citet{tosh2022simple}.  Proofs are in
\cref{app:oracle,app:membership}.

\paragraph{Motivation.}
The algorithm of \cref{sec:upper} needs weighted ERM over pairs of a group
and a hypothesis.  Separate oracles for groups and for hypotheses are known
to find only local optima of this joint problem, by alternating maximization
\citep{globusharris2022algorithmic,deng2024groupwise}, and a global optimum
can be hard to find, as we now restate.

\begin{restatement}{\Cref{lem:oracle}}
There is a polynomial-time map that takes a graph on $n$ vertices to classes
$\cG$ and $\cH$ on a finite input space, each of VC dimension at most $n$
and each admitting weighted ERM in linear time, together with a query to a
weighted ERM oracle for $\Fp$ whose optimal value determines the size of the
maximum cut of the graph.  Consequently, unless $\mathrm{P}=\mathrm{NP}$, no
polynomial-time algorithm implements weighted ERM over $\Fp$ using separate
weighted ERM oracles for $\cG$ and $\cH$.
\end{restatement}

The proof shows that weighted ERM over the pair class can express any
bilinear objective $\sum_rw_rg(x_r)h(x_r)$ over pairs, and builds classes in
which the group selects a sign for each vertex of the graph and the
hypothesis is forced to copy these signs, so that the bilinear objective
counts the cut edges, while each class on its own decomposes across vertices
and admits linear-time ERM.  The full proof is in \cref{app:oracle}.

\Cref{lem:oracle} concerns exact optimization over pairs: it does not rule
out a multi-group learner that uses only separate oracles and never solves
this problem, and it is not a lower bound on the number of oracle calls of
learners with unlimited computation.  It does mean that learning with
separate oracles needs a different algorithm, and the rest of this section
gives one under a natural assumption on the groups.

\begin{definition}[Bounded membership]
\label{def:membership}
For an input $x$, let $I(x)=\{g\in\cG:g(x)=1\}$ be the set of groups
containing $x$.  The class $\cG$ has \emph{membership at most $k$} if
$|I(x)|\le k$ for every $x\in\cX$.  \emph{Memberships can be listed} if, for
each input $x$, the list $I(x)$ can be computed in time polynomial in $k$
and the length of $x$, with each listed group given by a name from which it
can be evaluated at any input.
\end{definition}

If $\cG$ shatters a set of $d_G\ge1$ points, then a fixed one of these points
belongs to groups realizing all $2^{d_G-1}$ labelings of the others, so
membership at most $k$ implies $d_G\le1+\log_2k$.

\begin{remark}[Listing memberships]
\label{rem:listing}
Listing is an access assumption about the groups, separate from the ERM
oracle for $\cH$, and it is not implied by a weighted ERM oracle for $\cG$.
It holds whenever the groups are described explicitly in terms of the input,
as for intersectional groups defined by agreement on a subset of $a$
attributes, where $I(x)$ consists of the $2^a$ subsets of the attribute
values of $x$, and for hierarchical groups, where $I(x)$ is the path of $x$
in the hierarchy.
\end{remark}

\begin{definition}[Group realizability, \citealp{ardeshir2026group}]
\label{def:group-realizable}
The distribution $\cD$ is \emph{group-realizable} with respect to $\cG$ and
$\cH$ if $Y=c^\star(X)$ almost surely for a function
$c^\star:\cX\to\{-1,+1\}$ such that, for every $g\in\cG$, some $h\in\cH$
satisfies $h(x)=c^\star(x)$ for $\cD_X$-almost every $x$ with $g(x)=1$.
\end{definition}

Group realizability does not require a single hypothesis to be correct
everywhere: different groups may have different correct hypotheses.  Under
the zero-one loss it gives $\inf_{h\in\cH}L(h;g)=0$ for every group, so the
multi-group excess risk of a predictor $f$ is
$\sup_g\Pr\{g(X)=1,\ f(X)\ne Y\}$.

\paragraph{The algorithms.}
Both algorithms find the groups that appear often in a reference sample,
train one hypothesis per such group with one ERM call over $\cH$, and
aggregate the trained hypotheses on fresh samples with a sleeping-experts
algorithm, in which the hypothesis $\widehat h_g$ is awake on the inputs in
$g$.  They differ only in the aggregation: sleeping Prod in the agnostic case
(\cref{alg:heavy-ag}) and a deterministic sleeping weighted majority under
group realizability (\cref{alg:heavy-re}).  Minimizing the loss over the
examples in a group is one weighted ERM call over $\cH$, since
$\ell(h(x),y)=\ell(-1,y)+\frac12(1+h(x))(\ell(+1,y)-\ell(-1,y))$ is affine in
$h(x)$: the total loss over the examples in $g$ is a constant minus a
weighted sum of the values of $h$.  With a $\xi$-approximate oracle, the
empirical conditional risk of $\widehat h_g$ is within $\xi/n_g$ of the
minimum, where $n_g$ is the number of examples in $g$, which is below the
statistical error in \cref{thm:heavy-agnostic}.  In the realizable case,
$\widehat h_g$ maximizes $\sum_iy_ih(x_i)$ over the examples $(x_i,y_i)$ in
$g$; a consistent hypothesis attains $n_g$ and every other hypothesis at most
$n_g-2$, so a $\xi$-approximate oracle with $\xi\le1$ still returns a
consistent hypothesis.  In both algorithms, $N_\gamma=|\cF_\gamma|$ and
$I_\gamma(x)=\{g\in\cF_\gamma:g(x)=1\}$ is the set of groups awake at $x$; the
aggregation makes no oracle calls, and the output is evaluated from the
stored hypotheses and weights.  Both outputs are represented by at most
$N_\gamma\le2k/\gamma$ hypotheses, all members of $\cH$, and the weights of
the $T$ rounds, which change at most $k$ times per round; no inputs or order
are stored.  At an input $x$, both use the list $I(x)$ and at most $k$
evaluations of hypotheses: a draw from the randomized output of
\cref{alg:heavy-ag} picks a round $t$ and draws from $\pi_t(x)$, and the
deterministic output of \cref{alg:heavy-re} takes $O(Tk)$ further
operations to compute the $T$ weighted majorities and their majority.

\begin{algorithm}[t]
\caption{$\textsc{HeavyGroupsAgnostic}(\gamma,m,T,\eta)$: agnostic
  multi-group learning with an ERM oracle for $\cH$}
\label{alg:heavy-ag}
\begin{algorithmic}[1]
\Require mass threshold $\gamma$; sample sizes $m,T$; learning rate
  $\eta\le1/2$; membership lists; weighted ERM oracle for $\cH$
\State draw $m$ inputs, list the groups containing each, and let $\cF_\gamma$
  be the groups listed at least $m\gamma/2$ times (\cref{lem:heavy})
\State draw $T$ labeled examples; for each $g\in\cF_\gamma$, let $\widehat h_g$
  minimize the loss over the examples in $g$
  \Comment{one oracle call per group}
\State set $w_{g,1}=1/N_\gamma$ for each $g\in\cF_\gamma$
\For{$t=1,\ldots,T$, on fresh labeled examples $(X_t,Y_t)$}
  \State let $\pi_t$ predict at $x$ the label $\widehat h_g(x)$ for $g$ drawn
    from $I_\gamma(x)$ with probability proportional to $w_{g,t}$ ($+1$ if
    $I_\gamma(x)=\emptyset$)
  \State observe $(X_t,Y_t)$, and let
    $\ell_t=\sum_{g\in I_\gamma(X_t)}w_{g,t}\,\ell(\widehat h_g(X_t),Y_t)\big/
    \sum_{g\in I_\gamma(X_t)}w_{g,t}$
  \State for $g\in I_\gamma(X_t)$, set
    $w_{g,t+1}=w_{g,t}\bigl(1+\eta(\ell_t-\ell(\widehat h_g(X_t),Y_t))\bigr)$;
    keep the other weights
\EndFor
\State \Return the randomized predictor $Q=\frac1T\sum_t\pi_t$
\end{algorithmic}
\end{algorithm}

\begin{algorithm}[t]
\caption{$\textsc{HeavyGroupsRealizable}(\gamma,m,T)$: group-realizable
  multi-group learning with an ERM oracle for $\cH$}
\label{alg:heavy-re}
\begin{algorithmic}[1]
\Require mass threshold $\gamma$; sample sizes $m,T$; membership lists;
  weighted ERM oracle for $\cH$
\State draw $m$ inputs, list the groups containing each, and let $\cF_\gamma$
  be the groups listed at least $m\gamma/2$ times (\cref{lem:heavy})
\State draw $T$ labeled examples; for each $g\in\cF_\gamma$, let $\widehat h_g$
  be consistent with the examples in $g$
  \Comment{one oracle call per group}
\State set $w_{g,1}=1$ for each $g\in\cF_\gamma$
\For{$t=1,\ldots,T$, on fresh labeled examples $(X_t,Y_t)$}
  \State let $\pi_t$ predict at $x$ the weighted majority label of
    $\{\widehat h_g(x):g\in I_\gamma(x)\}$ with weights $w_{g,t}$, ties broken toward
    $+1$ ($+1$ if $I_\gamma(x)=\emptyset$)
  \State observe $(X_t,Y_t)$; for $g\in I_\gamma(X_t)$: if $\widehat h_g(X_t)\ne Y_t$,
    set $w_{g,t+1}=w_{g,t}/2$; else if $\pi_t(X_t)\ne Y_t$, set
    $w_{g,t+1}=\frac32w_{g,t}$; else keep $w_{g,t}$
\EndFor
\State \Return the deterministic predictor
  $f(x)=\mathrm{Maj}\{\pi_1(x),\ldots,\pi_T(x)\}$, ties broken toward $+1$
\end{algorithmic}
\end{algorithm}

The main result of the section is \cref{thm:membership}, which we restate.

\begin{restatement}{\Cref{thm:membership}}
Let $\varepsilon,\delta\in(0,1)$, and suppose that every input belongs to
at most $k$ groups and that the groups containing a given input can be
listed (\cref{def:membership}).
\begin{enumerate}[label=(\alph*),leftmargin=2em]
\item \emph{Agnostic.} There is an algorithm (\cref{alg:heavy-ag}) that, given
  $\varepsilon$, $\delta$, and a weighted ERM oracle for $\cH$, uses
  $\wtO((d_H+\log(k/\delta))/\varepsilon^{2})$ samples and at most
  $2k/\varepsilon$ oracle calls, and with probability at least $1-\delta$
  outputs a randomized predictor with multi-group excess risk at most
  $\varepsilon$.
\item \emph{Group-realizable.} If $\ell$ is the zero-one loss and $\cD$ is
  group-realizable (\cref{def:group-realizable}), then there is an algorithm
  (\cref{alg:heavy-re}) that, given $\varepsilon$, $\delta$, and a weighted
  ERM oracle for $\cH$, uses
  $\wtO((d_H+\log(k/\delta))/\varepsilon)$ samples and at most
  $2k/\varepsilon$ oracle calls, and with probability at least $1-\delta$
  outputs a deterministic predictor with multi-group excess risk at most
  $\varepsilon$.
\end{enumerate}
\end{restatement}

\paragraph{Proof outline.}
\Cref{sec:heavy} shows that at most $k/\gamma$ groups have probability at
least $\gamma$ and that the first step of \cref{alg:heavy-ag,alg:heavy-re} finds all of them
(\cref{lem:heavy}).  \Cref{sec:aggregate} bounds, for every such group, the
excess risk of the agnostic output (\cref{thm:heavy-agnostic}) and the error
of the realizable output (\cref{thm:heavy-realizable}), for a general
threshold $\gamma$.  At the end of the section, we prove
\cref{thm:membership} by taking $\gamma=\varepsilon$: groups of smaller
probability satisfy the guarantee trivially.

\subsection{Finding the heavy groups}
\label{sec:heavy}

Bounded membership limits the number of groups of non-negligible probability,
and these groups appear often in a reference sample.  Recall that
$p_g=\Pr\{g(X)=1\}$.

\begin{lemma}[Heavy groups]
\label{lem:heavy}
If $\cG$ has membership at most $k$, then at most $k/\gamma$ groups have
$p_g\ge\gamma$.  Let $U_1,\ldots,U_m$ be independent draws from $\cD_X$, and
let $\cF_\gamma$ be the set of groups that belong to at least $m\gamma/2$ of
the lists $I(U_1),\ldots,I(U_m)$.  Then $|\cF_\gamma|\le2k/\gamma$ always,
and if $m\ge(8/\gamma)\log(k/(\gamma\zeta))$, then with probability at least
$1-\zeta$, $\cF_\gamma$ contains every group with $p_g\ge\gamma$.
\end{lemma}

\begin{proof}
For any finite set of distinct groups, $\sum_gp_g=\E\sum_gg(X)\le k$, which
gives the first claim.  The lists contain at most $km$ entries in total, so
at most $2k/\gamma$ groups appear in at least $m\gamma/2$ of them.  A group with
$p_g\ge\gamma$ appears in a binomial number of lists with mean
$mp_g\ge m\gamma$, which is below $m\gamma/2$ with probability at most
$e^{-m\gamma/8}$ by the Chernoff bound; a union bound over the at most
$k/\gamma$ such groups completes the proof.
\end{proof}

\subsection{Training and aggregating one hypothesis per group}
\label{sec:aggregate}

One ERM call over $\cH$ per heavy group gives a hypothesis that is nearly
optimal on that group, and a sleeping-experts aggregation of these
hypotheses competes with each of them on its own group.  For a group $g$, a
randomized predictor $Q$, and a hypothesis $h$, write
$R_g(Q)=L(Q;g)/p_g$ and $R_g(h)=L(h;g)/p_g$ for the conditional risks, and
$R_g^*=\inf_{h\in\cH}R_g(h)$.

\begin{theorem}[Agnostic guarantee]
\label{thm:heavy-agnostic}
Suppose $\cG$ has membership at most $k$ and memberships can be listed, and
run \cref{alg:heavy-ag} with
$m\ge(8/\gamma)\log(3k/(\gamma\delta))$ and $\eta\le1/2$.  Let
$\Phi=d_H\log(eT)+\log(24k/(\gamma\delta))$.  With probability at least
$1-\delta$, for every group with $p_g\ge\gamma$,
\[
  L(Q;g)-\inf_{h\in\cH}L(h;g)
  \le C\Bigl[\sqrt{\frac{p_g\Phi}T}+\frac{\Phi}{\eta T}+\eta\,p_g\Bigr].
\]
The algorithm makes at most $2k/\gamma$ weighted ERM calls over $\cH$.
\end{theorem}

\begin{theorem}[Group-realizable guarantee]
\label{thm:heavy-realizable}
Suppose, in addition, that the loss is the zero-one loss and $\cD$ is
group-realizable (\cref{def:group-realizable}), and run \cref{alg:heavy-re} with
$m\ge(8/\gamma)\log(3k/(\gamma\delta))$.  With
$\Phi=d_H\log(eT)+\log(24k/(\gamma\delta))$, with probability at least
$1-\delta$, for every group with $p_g\ge\gamma$,
\[
  \Pr\{g(X)=1,\ f(X)\ne Y\}\le\frac{C\Phi}{T}.
\]
The predictor $f$ is deterministic, and the algorithm makes at most
$2k/\gamma$ weighted ERM calls over $\cH$.
\end{theorem}

\paragraph{Proof outline.}
For \cref{thm:heavy-agnostic}, by \cref{lem:heavy}, every group with $p_g\ge\gamma$ is in $\cF_\gamma$.
Conditionally on $\cF_\gamma$, each $\widehat h_g$ is ERM on about $Tp_g$
samples from the conditional distribution on $g$, so
$R_g(\widehat h_g)-R_g^*\lesssim\sqrt{\Phi/(Tp_g)}$.  The weights of sleeping
Prod sum to one at all times, which gives a pathwise regret bound against
each $\widehat h_g$ on the rounds when $g$ is awake.  Freedman's inequality
\citep{freedman1975tail} converts it to a bound on the population group
losses of the predictors $\pi_t$.  Since the loss of a randomized
prediction is linear in its distribution, $L(Q;g)$ is the average of these
losses.

For \cref{thm:heavy-realizable}, each $\widehat h_g$ is consistent with the examples in $g$, so by the
realizable uniform convergence bound it errs on a set of probability
$O(\Phi/T)$ within $g$.  On the rounds when $g$ is awake,
$\widehat h_g$ makes $O(\Phi)$ mistakes, and the weighted majority makes at
most a constant times as many plus $O(\log N_\gamma)$.  A relative
online-to-batch bound, based on an exponential supermartingale, turns the
number of mistakes of the predictors $\pi_t$ on $g$ into their population error on $g$, and the majority vote at most doubles the average error.  The full proof is in \cref{app:membership}.

We now prove the main result of the section.

\begin{proof}[Proof of \cref{thm:membership}]
In both parts, run \cref{alg:heavy-ag} or \cref{alg:heavy-re} with $\gamma=\varepsilon$ and
$m=\lceil(8/\varepsilon)\log(3k/(\varepsilon\delta))\rceil$.  A group with
$p_g<\varepsilon$ needs no guarantee: since the loss lies in $[0,1]$, every
$h$ has $L(Q;g)-L(h;g)\le L(Q;g)\le p_g<\varepsilon$, and likewise for $f$.

\emph{Part~(a).}  Take $\eta=\varepsilon/(4C)$ and $T$ the smallest integer
with $T\ge C'\Phi/\varepsilon^2$, for a large enough constant $C'$ (recall
that $\Phi$ grows logarithmically with $T$).  For a group with
$p_g\ge\varepsilon$, each term of the bound of \cref{thm:heavy-agnostic} is
then at most $\varepsilon/3$.

\emph{Part~(b).}  Take $T$ the smallest integer with
$T\ge C'\Phi/\varepsilon$.  Under group realizability
$\inf_hL(h;g)=0$, so for a group with $p_g\ge\varepsilon$,
\cref{thm:heavy-realizable} gives
$L(f;g)-\inf_hL(h;g)=\Pr\{g(X)=1,\ f(X)\ne Y\}\le C\Phi/T\le\varepsilon$.

In both parts the algorithm uses $m+2T$ samples, which is of the stated
order, and at most $2k/\varepsilon$ oracle calls.  For a general threshold
$\gamma$, dividing the bounds of \cref{thm:heavy-agnostic,thm:heavy-realizable}
by $p_g\ge\gamma$ gives the conditional guarantees stated after
\cref{thm:membership}.
\end{proof}

\section{Discussion}
\label{sec:discussion}

Multi-group learning asks one predictor to compete with a possibly different
hypothesis on each of many overlapping groups.  We view it as a game against
a pair player and implement the player with the hybrid learner over a class
that encodes the four possible outputs of a pair, using random symbols.  A
voting rule keeps the player's payoff near zero, and a martingale argument
converts the online guarantee into a batch one.  The resulting algorithm
uses $\wtO((d_G+d_H+\log(1/\delta))/\varepsilon^2)$ samples, which is
optimal up to polylogarithmic factors among all learners, and its output can
be made deterministic on ordered input spaces.  The oracle it needs
optimizes over pairs, and such an oracle cannot in general be implemented
from separate oracles for the two classes.  When every input belongs to few
groups and these can be listed, however, only few groups matter, and an
ERM oracle for the hypothesis class suffices, with the optimal agnostic rate,
a dependence of only $1/\gamma$ on the probability of the groups, and a
deterministic fast-rate predictor under group realizability.

Several questions remain.  First, can the optimal rate be attained with only
separate ERM oracles for $\cG$ and $\cH$ under assumptions weaker than
bounded membership with listing, or is there an oracle lower bound for such
learners?  Second, for groups of probability at least $\gamma$,
\cref{cor:conditional} uses $\wtO((d+\log(1/\delta))/(\gamma^2\varepsilon^2))$
samples, whereas for finite classes the enumerating algorithm of
\citet{tosh2022simple} adapts to the size of each group and needs only a
factor $1/\gamma$; can such an adaptive guarantee be obtained for VC classes
with polynomially many oracle calls?  Third, our algorithm makes
$\wtO((d+\log(1/\delta))^2/\varepsilon^4)$ oracle calls and stores as many
pairs; can both be reduced?  Fourth, our deterministic predictor uses a computable order on $\cX$ and
stored sample inputs (\cref{rem:representation}); can a deterministic
predictor be obtained, and represented by members of $\cG$ and $\cH$ alone,
without them, and, under bounded membership, can the
agnostic predictor of \cref{thm:membership} be made deterministic?  Finally, in the group-realizable setting
of \citet{ardeshir2026group}, can the fast rate $1/(\gamma\varepsilon)$ be
attained oracle-efficiently for general VC classes of groups?

\paragraph{AI disclosure.}
The author realized, upon coming across the result of
\citet{buzaglo2026oracle}, that it implies oracle-efficient bounds for
multi-group learning.  ChatGPT and Claude were then used to help with the
writing and organization of the paper.  The author takes full
responsibility for its content.


\appendix

\section{Proofs for Section~\texorpdfstring{\ref*{sec:upper}}{3}}
\label{app:upper}

This appendix proves \cref{thm:online,thm:batch}.  Throughout, $C$ denotes an
absolute constant whose value may change between occurrences.

\begin{lemma}[Pair regret from the hybrid learner]
\label{lem:pair-regret}
Under the conditions of \cref{thm:online}, with probability at least
$1-3\zeta/4$,
\[
  \sup_{g\in\cG,\,h\in\cH}\sum_{t=1}^T\Bigl(\rho_t(g,h)
  -\E_{(g',h')\sim q_t}\rho_t(g',h')\Bigr)
  \le2\Reg\Bigl(T,6d,\frac\zeta4\Bigr)
  +C\sqrt{T\bigl(d\log(eT)+\log(1/\zeta)\bigr)} .
\]
\end{lemma}

\begin{proof}
Write $s=(g,h)$, $F_s=F_{g,h}$, and
\[
  \Omega_t(s)=Z_tF_s(U_t),\qquad
  \overline\Omega_t=\E_{s\sim q_t}\Omega_t(s)=Z_t\bF_t(U_t),\qquad
  S_t=\sum_{a\in\cA}\Gamma_t(a).
\]
The symbol $A_t$ is uniform and independent of $q_t$, $X_t$, and the table
$\Gamma_t$, so by \cref{lem:encoding}, writing $\E_{A_t}$ for the expectation
over $A_t$ alone,
$\E_{A_t}\Omega_t(s)=\frac12\rho_t(s)-\frac14S_t$ and
$\E_{A_t}\overline\Omega_t=\frac12\E_{s'\sim q_t}\rho_t(s')-\frac14S_t$.
Hence
\[
  \sum_t\bigl(\rho_t(s)-\E_{q_t}\rho_t\bigr)
  =2\sum_t\bigl(\Omega_t(s)-\overline\Omega_t\bigr)
  -2\sum_t\bigl(\Omega_t(s)-\E_{A_t}\Omega_t(s)\bigr)
  +2\sum_t\bigl(\overline\Omega_t-\E_{A_t}\overline\Omega_t\bigr).
\]
The inputs $U_t=(X_t,A_t)$ are independent draws from
$\cD_X\otimes\mathrm{Unif}(\cA)$, and the feedback $Z_t=\Gamma_t(A_t)$ does
not depend on the round draw, so by \cref{thm:hybrid} with
$\VC(\Fp)\le6d$ (\cref{lem:pair-vc}), with probability at least
$1-\zeta/4$, the first sum is at most $\Reg(T,6d,\zeta/4)$ uniformly over
$s$.  The last sum is a martingale with increments bounded by $2$ in
absolute value, so by Azuma's inequality \citep{azuma1967weighted} it is at
most $C\sqrt{T\log(1/\zeta)}$ with probability at least $1-\zeta/4$.

For the middle sum, condition on $X_1,\ldots,X_T$, which are independent of
the symbols.  The quantity $\Omega_t(s)-\E_{A_t}\Omega_t(s)$ depends on $s$
only through the labels $(g(X_t),h(X_t))_{t\le T}$, of which there are at most
$\Pi_{\cG}(T)\Pi_{\cH}(T)\le(eT)^d$ possibilities.  For each fixed labeling,
the terms are martingale differences bounded by $2$, because the table
$\Gamma_t$ is chosen without seeing $A_t$.  Azuma's inequality and a union
bound over the labelings give, with probability at least $1-\zeta/4$,
$\sup_s|\sum_t(\Omega_t(s)-\E_{A_t}\Omega_t(s))|
\le C\sqrt{T(d\log(eT)+\log(1/\zeta))}$.  Combining the three bounds proves
the lemma.
\end{proof}

\begin{lemma}[Voting keeps the mixture's payoff small]
\label{lem:voting}
Under the conditions of \cref{thm:online}, with probability at least
$1-\zeta/4$,
\[
  \sum_{t=1}^T\E_{(g,h)\sim q_t}\rho_t(g,h)
  \le CT\sqrt{\frac{\log(32T/\zeta)}{M}}+C\sqrt{T\log(16/\zeta)} .
\]
\end{lemma}

\begin{proof}
Fix a round and suppress $t$.  For $z\in\{-1,+1\}$, let
$\mu_z=\E_{(g,h)\sim q}[g(X)\mathbf 1\{h(X)=z\}]$ and
$\widehat\mu_z=\frac1M\sum_ig^{(i)}(X)\mathbf 1\{h^{(i)}(X)=z\}$, and let
$\mu=\mu_{+1}+\mu_{-1}$ and $\widehat\mu=\widehat\mu_{+1}+\widehat\mu_{-1}$.
The drawn pairs are independent of each other and of $X$ given the past, so
Hoeffding's inequality and a union bound over the two labels and the $T$
rounds give, with probability at least $1-\zeta/8$,
$|\mu_z-\widehat\mu_z|\le b=\sqrt{\log(32T/\zeta)/(2M)}$ for all rounds and
labels.  For a label $a$ and an outcome $y$, the payoff of predicting $a$
against the mixture is
$\E_{(g,h)\sim q}[g(X)(\ell(a,y)-\ell(h(X),y))]=\mu\,\ell(a,y)-\sum_z\mu_z\ell(z,y)$,
and replacing $\mu$ and $\mu_z$ by $\widehat\mu$ and $\widehat\mu_z$ changes
it by at most $|\mu-\widehat\mu|+\sum_z|\mu_z-\widehat\mu_z|\le4b$.
By~\eqref{eq:vote-zero}, the empirical version has expectation zero over
$\Yhat\sim\phat(\cdot\mid X)$, for every $y$.  Since $Y_t$ does not depend on
the prediction, the conditional expectation of $\E_{q_t}\rho_t$ over the
round's prediction is at most $4b$, and the deviations from these conditional
expectations form a martingale with increments bounded by $2$.  Azuma's
inequality bounds their sum by $C\sqrt{T\log(16/\zeta)}$ with probability at
least $1-\zeta/8$.  Summing over the rounds proves the lemma.
\end{proof}

\begin{proof}[Proof of \cref{thm:online}]
Since
$\sup_s\sum_t\rho_t(s)\le\sup_s\sum_t(\rho_t(s)-\E_{q_t}\rho_t)+\sum_t\E_{q_t}\rho_t$,
\cref{lem:pair-regret,lem:voting} give the bound with probability at least
$1-\zeta$.  With $M=T$, the last term is $C\sqrt{T\log(32T/\zeta)}$.  Each
round draws $M$ pairs and makes one round draw, each with one weighted ERM
call on at most $T$ points.
\end{proof}

\begin{proof}[Proof of \cref{thm:batch}]
Realize the prediction on round $t$ by a uniform $R_t$ independent of
everything else, as $\Yhat_t=\textsc{Draw}(\phat_t(\cdot\mid X_t),R_t)$ with
$\textsc{Draw}(p,r)=+1$ if $r\le p(+1)$ and $-1$ otherwise.  Let
$z_t(x,y,r)=\ell(\textsc{Draw}(\phat_t(\cdot\mid x),r),y)\in[0,1]$, which is
determined by the drawn pairs of round $t$, and hence by the first $t-1$
samples and randomness independent of the later ones.  Then
$g(X_t)z_t(W_t)=g(X_t)\ell(\Yhat_t,Y_t)$, and
$\E_W[g(X)z_t(W)]=L(\phat_t;g)$, the group loss of the randomized predictor
$\phat_t$.  By linearity, $L(\Qbar;g)=\frac1n\sum_tL(\phat_t;g)$.

Let $b_n=C\sqrt{(d\log(en)+\log(8/\delta))/n}$.  By \cref{lem:predictable}
with $\zeta=\delta/4$, with probability at least $1-\delta/4$, for every $g$,
$L(\Qbar;g)\le\frac1n\sum_tg(X_t)\ell(\Yhat_t,Y_t)+b_n$.  By
\cref{lem:comparators} with $\zeta=\delta/4$, with probability at least
$1-\delta/4$, for every $g$ and $h$,
$L(h;g)\ge\frac1n\sum_tg(X_t)\ell(h(X_t),Y_t)-b_n$.  Subtracting,
\[
  L(\Qbar;g)-L(h;g)\le\frac1n\sum_t\rho_t(g,h)+2b_n .
\]
The samples are independent, so the conditions of \cref{thm:online} hold;
with $\zeta=\delta/2$ and $M=T=n$, with probability at least $1-\delta/2$,
the online multi-group regret is at most
\[
  C\log^2(2n)\sqrt{n\bigl(d\log(en)+\log(8/\delta)\bigr)} .
\]
Dividing by $n$, adding $2b_n$, and taking the supremum over $g$ and $h$
proves the bound, with total failure probability at most
$\delta/4+\delta/4+\delta/2=\delta$.  Choosing $n$ as the smallest integer for which the bound is at
most $\varepsilon$ gives $n=\wtO((d+\log(1/\delta))/\varepsilon^2)$, and
there are $n(M+1)=n(n+1)$ oracle calls.
\end{proof}

\section{Proof for Section~\texorpdfstring{\ref*{sec:det}}{4}}
\label{app:det}

\begin{proof}[Proof of \cref{thm:det}]
Work on the events of \cref{lem:coverage}, of \cref{lem:budget}, and of
\cref{lem:cell-tests} (applied below), whose failure probabilities total at most
$4\delta_0$; we condition on $Q$ and on the confidence and partition samples
when applying \cref{lem:cell-tests}, so that $Q'$, $\cP$, and the tests are
fixed and the seeds are independent.

Let
$\varphi(x,v)=0$ for $x\in\mathcal B$ and $\varphi(x,v)=v\Delta(x)/2$ otherwise.
For $x\in\mathcal B$, both $Q'_x$ and $f(x)$ are the point mass at $-\sgn(\Dhat_x)$,
so $f(x)=m_{Q'}(x)$ and the integrand in~\eqref{eq:gap} vanishes there.
Hence, by~\eqref{eq:gap} with $Q'$ in place of $Q$,
\[
  L(f;g)-L(Q';g)
  =\E\bigl[g(X)\varphi(X,f(X))\bigr]-\E\Bigl[g(X)\sum_vQ'_X(v)\varphi(X,v)\Bigr].
\]
Let $\cT=\{(x,v)\mapsto2g(x)-1:g\in\cG\}\cup\{(x,v)\mapsto1\}$, which has VC
dimension at most $d_G+1$.  Since $g=\frac12((2g-1)+1)$, the right side is
the average of two rounding errors of \cref{lem:cell-tests} for tests in
$\cT$.  For the weights, $|\varphi(x,v)|\le|\Delta(x)|/2\le2r_x$ for
$x\in O\setminus\mathcal B$ by the discussion after~\eqref{eq:decisive},
$\varphi=0$ on $\mathcal B$, and $|\varphi|\le1/2$ elsewhere, so we may take
$\varphi_{\{x\}}=2r_x$ for $x\in O$ and $\varphi_R=1/2$ for $R\in\Prest$,
which gives $V_\varphi\le4V_{\cP}$.  \Cref{lem:cell-tests} and
\cref{lem:budget} then give, for every $g$,
\begin{align*}
  L(f;g)-L(Q';g)&\le C\sqrt{V_{\cP}\bigl(d_G+1+\log(2/\delta_0)\bigr)}\\
  &\le C\sqrt{\Bigl(\frac{\log(e\ncf/\delta_0)}{\ncf}
  +\frac{\log(e\npt/\delta_0)}{\npt}\Bigr)\bigl(d_G+1+\log(2/\delta_0)\bigr)} .
\end{align*}
Combining with~\eqref{eq:decisive} proves the bound.  For the sample sizes,
$\ncf=\npt=\lceil C'(d_G+1+\log(2/\delta_0))\varepsilon^{-2}
\log(e(d_G+1+\log(2/\delta_0))/(\varepsilon\delta_0))\rceil$ suffice for a
large enough absolute constant $C'$.
\end{proof}

\section{Proof of the lower bound}
\label{app:lower}

We use two standard facts about the Bhattacharyya coefficient
$\BC(P,P')=\sum_\omega\sqrt{P(\omega)P'(\omega)}$ of distributions on a
finite set.  First, $\TV(P,P')\le\sqrt{1-\BC(P,P')^2}$, since by the
Cauchy--Schwarz inequality
\[
  \sum_\omega|P-P'|=\sum_\omega\bigl|\sqrt P-\sqrt{P'}\bigr|
  \bigl(\sqrt P+\sqrt{P'}\bigr)\le\sqrt{(2-2\BC)(2+2\BC)} .
\]
Second, $\sum_\omega\min\{P(\omega),P'(\omega)\}\ge\frac12\BC(P,P')^2$, since
\[
  \BC^2=\Bigl(\sum_\omega\sqrt{\min\{P,P'\}\max\{P,P'\}}\Bigr)^2
  \le\sum_\omega\min\{P,P'\}\sum_\omega\max\{P,P'\}
  \le2\sum_\omega\min\{P,P'\} .
\]  The coefficient
is multiplicative over product distributions.

\begin{proof}[Proof of \cref{prop:lower}]
Let $\beta=128\varepsilon\le1/2$, and let $x_1,\ldots,x_m$ be shattered by
$\cC$.  For $\sigma\in\{-1,+1\}^m$, let $\cD_\sigma$ be the distribution
with inputs uniform on $\{x_1,\ldots,x_m\}$ and label equal to $\sigma_i$ at
$x_i$ with probability $(1+\beta)/2$.  Some $c_\sigma\in\cC$ has
$c_\sigma(x_i)=\sigma_i$ for all $i$, and it minimizes the risk under
$\cD_\sigma$.  For a randomized predictor $Q$, let $e_i=Q_{x_i}(-\sigma_i)$.
Under the zero-one loss, the conditional loss of $Q$ at $x_i$ exceeds that of
$\sigma_i$ by $\beta e_i$, so $L(Q)-L(c_\sigma)=\frac\beta m\sum_ie_i$.

\emph{Wrong signs are rare.}  On success,
$\sum_ie_i\le\varepsilon m/\beta=m/128$.  Let $\widehat\sigma_i$ be the
label to which $Q_{x_i}$ gives probability at least $1/2$, with ties broken
toward $+1$; if $\widehat\sigma_i\ne\sigma_i$ then $e_i\ge1/2$, so
$\widehat\sigma$ differs from $\sigma$ in at most $2\sum_ie_i\le m/64$
coordinates.

\emph{Assouad's lemma.}  For a single labeled sample from $\cD_\sigma$ and
from $\cD_{\sigma'}$, where $\sigma'$ differs from $\sigma$ in one coordinate,
$\BC=1-\frac1m\bigl(1-\sqrt{1-\beta^2}\bigr)\ge1-\beta^2/m$.  Suppose that
the learner uses $n\le m/(64\beta^2)$ samples.  For $\sigma$ and $\sigma'$ differing in one coordinate, the
distributions $P_\sigma^n$ and $P_{\sigma'}^n$ of the sample (together with
the learner's randomness) satisfy $\BC\ge(1-\beta^2/m)^n\ge e^{-2n\beta^2/m}$,
since $\beta^2/m\le1/4$, so
$\TV\le\sqrt{1-e^{-4n\beta^2/m}}\le2\beta\sqrt{n/m}\le1/4$.  By Assouad's
lemma \citep{assouad1983deux,yu1997assouad}, some $\sigma$ has
$\E_\sigma|\{i:\widehat\sigma_i\ne\sigma_i\}|\ge\frac m2(1-\frac14)=\frac{3m}8$.
On the other hand, the learner succeeds with probability at least
$1-\delta\ge15/16$, and then $\widehat\sigma$ has at most $m/64$ wrong
coordinates, so the expectation is at most $m/64+m/16<3m/8$, a contradiction.
Hence $n>m/(64\beta^2)=m/(2^{20}\varepsilon^2)$.

\emph{The term $\log(1/\delta)$.}  Apply the construction to the single
shattered point $x_1$ ($m=1$).  On success, $\beta e_1\le\varepsilon$, so
$e_1\le1/128$ and $\widehat\sigma_1=\sigma_1$.  The learner therefore yields
a test between $\sigma_1=+1$ and $\sigma_1=-1$ whose error under each is at
most $\delta$.  For any test, the two errors sum to at least
$\sum_\omega\min\{P_+^n(\omega),P_-^n(\omega)\}$, which is at least
$\frac12\BC(P_+^n,P_-^n)^2=\frac12(1-\beta^2)^n\ge\frac12e^{-2n\beta^2}$.
Hence $2\delta\ge\frac12e^{-2n\beta^2}$, so
$n\ge\log(1/(4\delta))/(2\beta^2)\ge\log(1/\delta)/(4\beta^2)$ for
$\delta\le1/16$, which is $\log(1/\delta)/(2^{16}\varepsilon^2)$.
\end{proof}

\section{Proof of \texorpdfstring{\cref{lem:oracle}}{Lemma 1.2}}
\label{app:oracle}

\begin{lemma}[Bilinear objectives]
\label{lem:bilinear-pair}
For weighted inputs $(x_r,w_r)$, the query to a weighted ERM oracle for
$\Fp$ with the points $(x_r,(1,+1))$ and $(x_r,(1,-1))$ of weights $w_r/2$
and $-w_r/2$ has objective $\sum_rw_r\,g(x_r)h(x_r)$ at $F_{g,h}$.
\end{lemma}

\begin{proof}
For $g(x)\in\{0,1\}$ and $h(x)\in\{-1,+1\}$,
\[
  F_{g,h}(x,(1,+1))=g(x)\bigl(1+h(x)\bigr)-1,\qquad
  F_{g,h}(x,(1,-1))=g(x)\bigl(1-h(x)\bigr)-1,
\]
so half their difference is $g(x)h(x)$.
\end{proof}

\begin{proof}[Proof of \cref{lem:oracle}]
Let the graph have vertices $[n]$ and edge set $E$, let $B_0=n+1$, and let
$a_{ii}=B_0$, $a_{ij}=-1$ if $\{i,j\}\in E$, and $a_{ij}=0$ otherwise.  For
$s,t\in\{-1,+1\}^n$, let $\Psi(s,t)=\sum_{i,j}a_{ij}s_it_j$.

For each ordered pair $(i,j)$ with $a_{ij}\ne0$, introduce two inputs
$x_{ij}^0$ and $x_{ij}^1$.  For $s\in\{-1,+1\}^n$, let
$g_s(x_{ij}^0)=(1+s_i)/2$ and $g_s(x_{ij}^1)=1$; for $t\in\{-1,+1\}^n$, let
$h_t(x_{ij}^0)=h_t(x_{ij}^1)=t_j$.  Let $\cG=\{g_s\}$ and $\cH=\{h_t\}$.
Each class has at most $2^n$ members, hence VC dimension at most $n$.  Give
$x_{ij}^0$ weight $2a_{ij}$ and $x_{ij}^1$ weight $-a_{ij}$; then
$2a_{ij}\frac{1+s_i}2t_j-a_{ij}t_j=a_{ij}s_it_j$, so the weighted bilinear
objective $\sum_rw_rg_s(x_r)h_t(x_r)$ equals $\Psi(s,t)$, and by
\cref{lem:bilinear-pair} it is the objective of a weighted ERM query for
$\Fp$ with $O(n+|E|)$ points.

Each class admits weighted ERM in linear time.  For $\cG$, the weighted sum
$\sum_rw_rg_s(x_r)$ is a constant plus a sum of terms, each depending on a
single coordinate $s_i$, so each $s_i$ can be chosen separately; the same
holds for $\cH$ and the coordinates $t_j$.

For fixed $s$, the coefficient of $t_j$ in $\Psi(s,t)$ is
$B_0s_j-\sum_{i:\{i,j\}\in E}s_i$, which has the sign of $s_j$ because $B_0$
exceeds every degree.  Hence $t=s$ maximizes $\Psi(s,\cdot)$, and
\[
  \max_{s,t}\Psi(s,t)=\max_s\Bigl(B_0n-2\sum_{\{i,j\}\in E}s_is_j\Bigr)
  =B_0n-2|E|+4\,\mathrm{MaxCut},
\]
since an edge has $s_is_j=-1$ exactly when it is cut.  The optimal value of
the query therefore determines the maximum cut, which is NP-hard to compute
\citep{karp1972reducibility}.  A polynomial-time algorithm that implements
weighted ERM over $\Fp$ from the separate oracles, which here run in linear
time, would compute it in polynomial time.
\end{proof}

\section{Proofs for Section~\texorpdfstring{\ref*{sec:membership}}{5}}
\label{app:membership}

Throughout, $\cF_\gamma$ is the set of retained groups, $N_\gamma=|\cF_\gamma|\le2k/\gamma$,
and we work on the event of \cref{lem:heavy} with $\zeta=\delta/3$, on which
every group with $p_g\ge\gamma$ belongs to $\cF_\gamma$.  The two labeled
samples are independent of the reference sample and of each other, so
conditionally on $\cF_\gamma$ the retained groups are fixed, and
conditionally on the first labeled sample the hypotheses $\widehat h_g$ are
fixed.  Let $\kappa_\delta=\log(12N_\gamma/\delta)$; since $N_\gamma\le2k/\gamma$,
$\kappa_\delta\le\Phi$.

\begin{proof}[Proof of \cref{thm:heavy-agnostic}]
\emph{Training.}  Fix $g\in\cF_\gamma$, and let $n_g$ be the number of examples
of the first labeled sample in $g$, a binomial variable with mean $Tp_g$.
Given $n_g$, these examples are independent draws from the conditional
distribution on $g$, and on any $n$ points the losses
$(x,y)\mapsto\ell(h(x),y)$, $h\in\cH$, take at most $\Pi_{\cH}(n)\le(en)^{d_H}$
value vectors.  By uniform convergence for this class
\citep{vapnik1971uniform} and the Chernoff bound $n_g\ge Tp_g/2$, each with
failure probability at most $\delta/(12N_\gamma)$, with probability at least
$1-\delta/(6N_\gamma)$,
$R_g(\widehat h_g)-R_g^*\le C\sqrt{\Phi/(Tp_g)}$; when $Tp_g<8\kappa_\delta$ the
Chernoff bound may fail, but then the right side exceeds $1$ for $C\ge3$,
and the bound holds trivially.  A union bound over $\cF_\gamma$ gives, with
probability at least $1-\delta/6$, for every $g\in\cF_\gamma$,
\begin{equation}
  \label{eq:heavy-train}
  p_g\bigl(R_g(\widehat h_g)-R_g^*\bigr)\le C\sqrt{p_g\Phi/T} .
\end{equation}

\emph{Aggregation: a pathwise bound.}  For $g\in I_\gamma(X_t)$, write
$\ell_{g,t}=\ell(\widehat h_g(X_t),Y_t)$.  Since $\ell_t$ is the weighted
average of $\ell_{g,t}$ over $I_\gamma(X_t)$, the update preserves
$\sum_{g\in I_\gamma(X_t)}w_{g,t}$, so the weights sum to $1$ at all times and each lies in
$(0,1]$.  Fix $g$.  Using $\log(1+v)\ge v-v^2$ for $|v|\le1/2$ and
$|\eta(\ell_t-\ell_{g,t})|\le1/2$,
\[
  0\ge\log w_{g,T+1}\ge-\log N_\gamma+\eta\sum_tg(X_t)(\ell_t-\ell_{g,t})
  -\eta^2\sum_tg(X_t),
\]
so
\begin{equation}
  \label{eq:heavy-path}
  \sum_{t=1}^Tg(X_t)\bigl(\ell_t-\ell_{g,t}\bigr)
  \le\frac{\log N_\gamma}\eta+\eta\sum_{t=1}^Tg(X_t).
\end{equation}

\emph{Aggregation: from the sample to the population.}  The predictor $\pi_t$
is fixed before round $t$, and $\ell_t$ is its expected loss at
$(X_t,Y_t)$, so the conditional expectation of $g(X_t)(\ell_t-\ell_{g,t})$
given the past is $p_g(R_g(\pi_t)-R_g(\widehat h_g))$.  The deviations are
martingale differences bounded by $2$ in absolute value, with conditional
variance at most $p_g$.  Freedman's inequality \citep{freedman1975tail} gives,
with probability at least $1-\delta/(6N_\gamma)$,
\[
  Tp_g\cdot\frac1T\sum_t\bigl(R_g(\pi_t)-R_g(\widehat h_g)\bigr)
  \le\sum_tg(X_t)(\ell_t-\ell_{g,t})+\sqrt{2Tp_g\kappa_\delta}+2\kappa_\delta,
\]
and Bernstein's inequality gives $\sum_tg(X_t)\le2Tp_g+2\kappa_\delta$ with
probability at least $1-\delta/(6N_\gamma)$.  The loss of a randomized prediction is
linear in its distribution, so $R_g(Q)=\frac1T\sum_tR_g(\pi_t)$.  Combining
with~\eqref{eq:heavy-path}, using $\log N_\gamma\le\Phi$, $\kappa_\delta\le\Phi$, and
$\eta\le1/2$, and taking a union bound over $\cF_\gamma$, with probability at
least $1-\delta/3$, for every $g\in\cF_\gamma$,
\[
  p_g\bigl(R_g(Q)-R_g(\widehat h_g)\bigr)
  \le C\Bigl[\frac{\Phi}{\eta T}+\eta p_g+\sqrt{\frac{p_g\Phi}T}\Bigr].
\]
Adding~\eqref{eq:heavy-train} and using
$L(Q;g)-\inf_hL(h;g)=p_g(R_g(Q)-R_g^*)$ proves the bound, with total failure
probability at most $\delta/3+\delta/6+\delta/3<\delta$ (each Freedman and
Bernstein event fails with probability at most $e^{-\kappa_\delta}\le\delta/(6N_\gamma)$).  The oracle calls
are the $N_\gamma\le2k/\gamma$ training calls.
\end{proof}

\begin{proof}[Proof of \cref{thm:heavy-realizable}]
\emph{Training.}  Fix $g\in\cF_\gamma$.  Some hypothesis agrees with
$c^\star$ almost surely on $g$, so with probability one some hypothesis is
consistent with the examples in $g$, and $\widehat h_g$ is consistent with them.
The functions $(x,y)\mapsto g(x)\mathbf 1\{h(x)\ne y\}$, $h\in\cH$, take at most
$(en)^{d_H}$ value vectors on $n$ points, so by the uniform convergence bound
for functions with zero empirical mean \citep{vapnik1971uniform}, with
probability at least $1-\delta/(6N_\gamma)$,
$\beta_g:=\Pr\{g(X)=1,\ \widehat h_g(X)\ne c^\star(X)\}\le C\Phi/T$.

\emph{Mistakes of the weighted majority.}  Let $W_t$ be the total weight.
When $\pi_t$ is correct, weights only decrease.  When $\pi_t$ errs, the awake
weight on the correct label is at most that on the wrong label; multiplying
the former by $3/2$ and halving the latter does not increase $W_t$.  Hence
$W_t\le N_\gamma$.  For fixed $g$, let $B_g$ be the number of rounds on which $g$ is
awake and $\widehat h_g$ errs, and $J_g$ the number on which $g$ is awake,
$\widehat h_g$ is correct, and $\pi_t$ errs.  Then
$(3/2)^{J_g}2^{-B_g}=w_{g,T+1}\le N_\gamma$, and the number $\nu_g$ of rounds on which
$g$ is awake and $\pi_t$ errs satisfies
\[
  \nu_g\le J_g+B_g\le\frac{\log N_\gamma}{\log(3/2)}
  +\Bigl(1+\frac{\log2}{\log(3/2)}\Bigr)B_g .
\]
Conditionally on the first sample, $B_g$ is a sum of $T$ independent
indicators with mean $\beta_g$ each, so by the Chernoff bound
$B_g\le2T\beta_g+3\kappa_\delta\le C\Phi$ with probability at least
$1-\delta/(6N_\gamma)$.

\emph{From mistakes to population error.}  Let
$V_t=g(X_t)\mathbf 1\{\pi_t(X_t)\ne Y_t\}\in\{0,1\}$ and
$e_t=\Pr\{g(X)=1,\ \pi_t(X)\ne c^\star(X)\}$, its conditional mean given the
past.  Since $\E[e^{-V_t}\mid\text{past}]=1-(1-e^{-1})e_t\le e^{-(1-e^{-1})e_t}$,
the process $\exp((1-e^{-1})\sum_{t\le s}e_t-\sum_{t\le s}V_t)$ is a nonnegative
supermartingale with initial value $1$, so by Markov's inequality, with
probability at least $1-\delta/(6N_\gamma)$,
$\sum_te_t\le(\nu_g+\kappa_\delta)/(1-e^{-1})\le C\Phi$.  Finally,
$\mathbf 1\{f(x)\ne c^\star(x)\}\le\frac2T\sum_t\mathbf 1\{\pi_t(x)\ne c^\star(x)\}$
pointwise, so $\Pr\{g(X)=1,\ f(X)\ne Y\}\le\frac2T\sum_te_t\le C\Phi/T$.  A
union bound over $\cF_\gamma$ and the three events gives total failure
probability at most $\delta/3+\delta/2<\delta$.
\end{proof}

\section{Proofs of the uniform convergence lemmas}
\label{app:uniform}

\begin{proof}[Proof of \cref{lem:predictable}]
Let $\cG_{1/n}\subseteq\cG$ be a $1/n$-cover of $\cG$ in $L_1(\cD_X)$, with
$\log|\cG_{1/n}|\le C(d_G+1)\log(en)$ \citep{haussler1995sphere}; it depends
only on $\cD_X$ and is used only in the analysis.  For a fixed
$g'\in\cG_{1/n}$, the terms $g'(X_t)z_t(W_t)-\E_W[g'(X)z_t(W)]$ are
martingale differences bounded by $1$, because $z_t$ is determined before
$W_t$ is drawn.  Azuma's inequality and a union bound over the cover give the
bound for every $g'\in\cG_{1/n}$ with probability at least $1-\zeta/2$.

For an arbitrary $g$, choose $g'\in\cG_{1/n}$ with
$\Pr\{g(X)\ne g'(X)\}\le1/n$.  Since $0\le z_t\le1$, the two empirical
averages differ by at most $\frac1n\sum_t\mathbf 1\{g(X_t)\ne g'(X_t)\}$, and
the two population averages by at most $1/n$.  On $m$ points, the
disagreement sets $\{g\ne g'\}$ have at most $\Pi_{\cG}(m)^2$ patterns, so by uniform
convergence for these sets \citep{vapnik1971uniform}, with probability at
least $1-\zeta/2$, every disagreement frequency exceeds its probability by at
most $C\sqrt{((d_G+1)\log(en)+\log(2/\zeta))/n}$.  Combining the two events
proves the lemma.
\end{proof}

\begin{proof}[Proof of \cref{lem:comparators}]
On any $m$ labeled points, the vector of values of
$(x,y)\mapsto g(x)\ell(h(x),y)$ is determined by the labelings of the inputs
by $g$ and $h$, so the class of these functions takes at most
$\Pi_{\cG}(m)\Pi_{\cH}(m)\le(em)^d$ value vectors on $m$ points, with values
in $[0,1]$.  The symmetrization argument behind the uniform convergence
theorem of \citet{vapnik1971uniform}, with a ghost sample of size $n$,
Hoeffding's inequality for the random swaps of matched points, and a union
bound over the at most $(2en)^d$ value vectors on $2n$ points, gives the
bound.
\end{proof}

\end{document}